\documentclass{article}
\usepackage{log_2024}						

\usepackage{booktabs}						
\usepackage{multirow}						
\usepackage{graphicx}						
\usepackage{duckuments}	
\usepackage{amsmath,amssymb,amsfonts}
\usepackage{algorithmic}
\usepackage{algorithm}
\usepackage{enumitem}
\usepackage{bbm}
\usepackage{subcaption}
\usepackage{url}            
\usepackage{booktabs}       
\usepackage{nicefrac}       
\usepackage{microtype}      
\usepackage{xcolor}

\usepackage[numbers,compress,sort]{natbib}	

\DeclareMathOperator*{\argmax}{arg\,max}
\DeclareMathOperator*{\argmin}{arg\,min}

\newtheorem{theorem}{Theorem}
\newtheorem{lemma}{Lemma}
\newtheorem{remark}{Remark}

\newtheorem{definition}{Definition}

\newcommand\af{\alpha_d}
\newcommand\bt{\beta_d}
\newcommand\ad{\log \left( \frac{\alpha_d (1-\beta_d)}{\beta_d (1-\alpha_d)} \right)}
\newcommand\bb{\log \left( \frac{1-\theta}{\theta} \right)}
\newcommand\HG{\{ HG_d \}_{d=2}^W}

\newenvironment{proof}{{\noindent\it Proof:}\quad}{\hfill $\square$\par}

\title{Matrix Completion with Hypergraphs: Sharp Thresholds and Efficient Algorithms}

\author[Z. Ma et al.]{%
Zhongtian Ma\\
Northwestern Polytechnical University \\
\email{mazhongtian@mail.nwpu.edu.cn}\And
Qiaosheng Zhang\\
Shanghai Artificial Intelligence Laboratory\\
\email{zhangqiaosheng@pjlab.org.cn}\And
Zhen Wang\thanks{Corresponding author.}\\
Northwestern Polytechnical University\\
\email{w-zhen@nwpu.edu.cn}
}

\begin{document}

\maketitle

\begin{abstract}
This paper considers the problem of completing a rating matrix based on sub-sampled matrix entries as well as observed social graphs and hypergraphs. We show that there exists a \emph{sharp threshold} on the sample probability for the task of exactly completing the rating matrix---the task is achievable when the sample probability is above the threshold, and is impossible otherwise---demonstrating a phase transition phenomenon. The threshold can be expressed as a function of the ``quality'' of hypergraphs, enabling us to  \emph{quantify} the amount of reduction in sample probability due to the exploitation of hypergraphs. This also highlights the usefulness of hypergraphs in the matrix completion problem. En route to discovering the sharp threshold, we develop a computationally efficient matrix completion algorithm that effectively exploits the observed graphs and hypergraphs. Theoretical analyses show that our algorithm succeeds with high probability as long as the sample probability exceeds the aforementioned threshold, and this theoretical result is further validated by synthetic experiments.  Moreover, our experiments on a real social network dataset (with both graphs and hypergraphs) show
that our algorithm outperforms other state-of-the-art matrix completion algorithms.\footnote{The source code for this article is available on \url{https://github.com/mztmzt/MCH_log}.}
\end{abstract}

\section{Introduction}\label{sec:intro}
Recommender systems are becoming increasingly popular as they provide personalized and tailored recommendations to users based on their preferences, interests, and actions~\citep{lu2012recommender, ricci2015recommender}. Relevant applications include online shopping, social media, and search engines~\citep{sivapalan2014recommender, guy2011social}. A commonly-used and well-known technique for recommender systems is \emph{low-rank matrix completion}, which aims to fill in missing values in a user-item matrix given the partially observed entries~\citep{keshavan2010matrix, ramlatchan2018survey}. To enhance the performance of recommender systems and to tackle the \emph{cold start problem} (i.e., recommending items to a new user who has not rated any items)~\citep{lika2014facing}, \emph{social network information} has widely been incorporated in many modern algorithms \citep{camacho2018social, sedhain2014social, zhao2015connecting}. 

Despite the impressive performance achieved by these algorithms, there has been a lack of theoretical insights into the usefulness of social network information in recommender systems, leaving the maximum possible gain due to social networks unknown. Recently, some works tried to address the aforementioned challenges from an information-theoretic perspective \citep{ahn2018binary, jo2021discrete, zhang2021community, zhang2022mc2g, elmahdy2020matrix, elmahdy2022optimal, suh2021use}, through investigating a matrix completion problem that consists of social graphs.
Ref.~\cite{ahn2018binary} theoretically revealed, for the first time, the gain due to social graphs by characterizing the minimum sample probability required for matrix completion. The follow-up works \citep{zhang2021community,zhang2022mc2g} further considered matrix completion with both social and item-similarity graphs, and \cite{elmahdy2020matrix} considered a more complicated scenario where social graphs with hierarchical structure is available. 

In addition to graph information, \emph{hypergraph} information is another type of information that is prevalent in social networks and is becoming an increasingly important resource for recommender systems. Hypergraphs, as the generalization of graphs, can capture  \emph{high-order relationships} among users, which better reflect the complex interactions of users in real scenarios~\citep{de2020social}. For example, friendships between users in social networks can be captured by graphs, but chat groups (as high-order relationships among group users) are usually represented by hyperedges in hypergraphs. While some prior works~\citep{zheng2018novel, li2013link, zhao2018learning} have leveraged hypergraph information (as part of the social network) into recommender systems and experimental evidences therein demonstrated the effectiveness, theoretical understandings of the benefit of hypergraph information are still lacking. This raises the question of interest in this paper:

\begin{center}
\emph{How much can the performance of recommender systems be improved by \\exploiting hypergraph information in social networks?}
\end{center}

To answer this question, we consider an abstraction of real recommender systems---a matrix completion problem that consists of a sub-sampled rating matrix as well as observed social graphs and hypergraphs. Our approach to quantify the gain due to hypergraph information is via investigating the interplay between  the ``quality'' of hypergraphs and the minimum sample probability required for the matrix completion task (to be detailed in Sections~\ref{sce:MCH_Guarantee} and~\ref{sec:lower_bound}).

As a first attempt to theoretically analyze matrix completion with hypergraphs, we consider a setting with $n$ users, $m$ items, and an $n \times m$  rating matrix. Each entry of the matrix is either $+1$ (like) or $-1$
(dislike). To reflect real scenarios where users are often clustered, we assume users are partitioned into $K$ disjoint clusters (where $K \ge 2$). Motivated by the \emph{homophily} phenomenon~\citep{mcpherson2001birds} in social sciences, we assume users in the same cluster have the same ratings over items. The learner observes three pieces of information: (i) a sub-sampled rating matrix with each entry being sampled  with sample probability $p$, and then potentially being flipped with probability $\theta$ to account for potential noise, (ii) a social graph generated via a celebrated random graph model with planted clusters---the stochastic block model (SBM)~\citep{holland1983stochastic}, and (iii) social hypergraphs generated via the hypergraph stochastic block model (HSBM)~\citep{kim2018stochastic}. The task is to achieve \emph{exact matrix completion} (i.e., complete the sub-sampled matrix without any error) using the observations. A more detailed description of our setting is provided in Section~\ref{sec:problem_state}, and a pictorial representation is presented in Figure~\ref{HG}.

\begin{figure}
  \centering
  \includegraphics[width=0.9\textwidth]{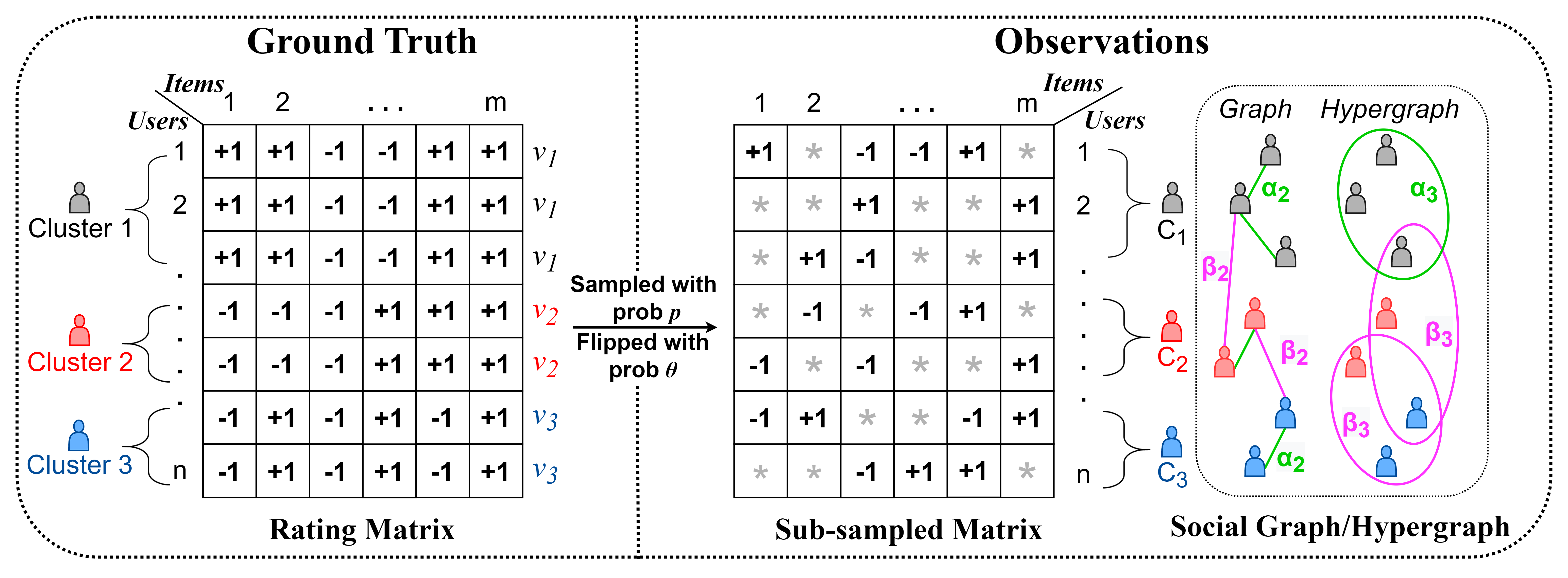}
  \caption{\small An illustration of the considered matrix completion problem. The goal is to exactly recover the rating matrix by exploiting the sub-sampled matrix, as well as the observed social graph and hypergraphs.}\label{HG}

\end{figure}

\paragraph{Main Contributions.} Our contributions are three-fold. 

First, we develop a computationally efficient matrix completion algorithm, named MCH (\underline{M}atrix \underline{C}ompletion with \underline{H}ypergraphs), that operates in three stages and can effectively leverage both social graphs and hypergraphs. It first adopt a \emph{spectral clustering method} on the social graph and hypergraphs to coarsely estimate the user clusters, then estimate users' ratings based on the observed sub-sampled rating matrix, and finally refine both the clusters and users' ratings in an iterative manner. Under the \emph{symmetric setting} wherein the $K$ clusters are of equal sizes (described in Section~\ref{sce:MCH_Guarantee}), we show that MCH achieves exact matrix completion with high probability as long as the sample probability exceeds a certain threshold presented in Theorem~\ref{The1}.

Second, we provide an \emph{information-theoretic lower bound} on the sample probability for the aforementioned matrix completion task (see Theorem~\ref{The2}). Under the symmetric setting, it matches the threshold in Theorem~\ref{The1}, thus showing that there exists a \emph{sharp threshold} on the value of the sample probability. This also demonstrates the optimality of our algorithm MCH in terms of the sample efficiency. Notably, the sharp threshold is a function of the ``quality'' of hypergraphs, by which one can quantify the gain due to hypergraph information in the matrix completion task. This gain is analyzed in detail in Section~\ref{sec:lower_bound}.


Third, we perform extensive experiments on synthetic datasets, and the results of these experiments further validate the theoretical guarantee of MCH. We then compare  MCH with other matrix completion algorithms on a semi-real dataset, which consists of a real social network with hypergraphs (the \emph{contact-high-school dataset} \citep{chodrow2021hypergraph, Mastrandrea-2015-contact}) and a synthetic rating matrix. Experimental results demonstrate the superior performance of MCH over other state-of-the-art algorithms.

\vspace{-4pt}
\paragraph{Related Works.} Many recommender systems have successfully used social network information, often relying on pairwise user relationships represented as graphs~\citep{eirinaki2018recommender, ma2008sorec}. However, real-world user interactions often involve high-order relationships that simple graphs can't capture. To better utilize social networks, recent studies have focused on hypergraphs~\cite{bu2010music, zheng2018novel}. For instance, ref.~\cite{bu2010music} used hypergraphs in a music recommender system, showing promising results. Additionally, some deep learning methods have integrated hypergraphs into graph neural networks to embed social network information~\citep{xia2022hypergraph, wei2022dynamic, yang2022multi, xia2022self}. Despite their success, the theoretical understanding of hypergraph benefits is still limited.

Recently, there has been a line of research devoted to quantifying the benefit of graph information in recommender systems, by analyzing a specific generative model for matrix completion.
Ref.~\cite{ahn2018binary} first proposed a matrix completion model in which a social graph (generated via the SBM) is available for exploitation, and revealed the gain due to graph information by characterizing the optimal sample probability for matrix completion.\footnote{As part of our work is inspired by ~\cite{ahn2018binary}, we provide a more detailed comparison in Appendix~\ref{compare}.} Ref.~\cite{zhang2021community, zhang2022mc2g} considered a more general scenario in which both the social and item graphs are observable, and designed a matrix completion algorithm that can fully utilize the information in the social and item graphs. Ref.~\cite{elmahdy2020matrix} showed that exploiting the hierarchical structure of social graphs yields a substantial gain for matrix completion compared to the work by~\citep{ahn2018binary}. These works are closely related to our work, but none of them has paid attention to the importance of hypergraph information in recommender systems. Moreover, due to the complicated structure of hypergraphs, theoretical analyses with respect to hypergraphs are arguably more challenging.  

 Our work is also closely related to \emph{community detection}, as achieving matrix completion in our problem requires detecting the communities/clusters of users based on the observed social graphs/hypergraphs. For graphs that are generated via the SBM (as assumed in this work), it has been shown~\citep{abbe2015exact,mossel2015consistency} that there exists a sharp   threshold for exact recovery of clusters. Similarly, the threshold for exact recovery of clusters in the HSBM has also been established~\citep{kim2018stochastic,zhang2022exact}. Moreover, our problem is also related to community detection with side-information~\citep{saad2018community, esmaeili2019exact, sima2021exact}, since the rating matrix in this work can be viewed as a special form of side-information for detecting the communities/clusters. 

\vspace{-4pt}

\paragraph{Notations.} For any positive integer $a$, let $[a] \triangleq \{1, 2, \ldots, a\}$. We use standard \emph{asymptotic notations}, including $O(.)$, $o(.)$, $\Omega(.)$, $\omega(.)$, and $\Theta(.)$, to describe the limiting behaviour of functions/sequences~\cite[Chapter~3.1]{leiserson1994introduction}. For an event $E$, we use $\mathbbm{1}\{E\}$ to denote the \emph{indicator function} that outputs $1$ if $\mathcal{E}$ is true and outputs $0$ otherwise.

\vspace{-4pt}
\section{Problem Statement}\label{sec:problem_state}
\vspace{-4pt}

\paragraph{Model.} Consider a rating matrix consisting of $n$ users and $m$ items. We assume users' ratings to items are either $+1$ or $-1$, which reflect ``like'' and ``dislike'' respectively. As observed in social science literature, people in real life are often clustered~\citep{borgatti2009network}, and people in the same cluster tend to have similar preferences (called \emph{homophily}~\citep{mcpherson2001birds}). To reflect these observations and to make the model as concise as possible, we assume the $n$ users are partitioned into $K$ disjoint clusters (where $K \ge 2$), and users in the same cluster have the same ratings to items.  To be concrete: 
\begin{itemize}[wide, labelwidth=!, labelindent=0pt]
    \item The $K$ clusters are denoted by $\{\mathcal{C}_1, \mathcal{C}_2, \ldots, \mathcal{C}_K\}$, where $\mathcal{C}_k \subset [n]$. These clusters are disjoint, i.e., for any $k_1,k_2 \in [K]$ such that $k_1 \ne k_2$, we have $\mathcal{C}_{k_1} \cap \mathcal{C}_{k_2}=\emptyset$. Moreover, $\cup_{k \in [K]} \mathcal{C}_k = [n]$.
    
    \item For users belonging to cluster $\mathcal{C}_k$ (where $k \in [K]$), their ratings to the $m$ items are represented by a length-$m$ vector $v_k \in \{+1,-1\}^m$, which is called the \emph{nominal rating vector} of cluster $\mathcal{C}_k$.
    
    \item We denote the \emph{rating matrix} to be completed as $R \in \{+1, -1\}^{n \times m}$, where the entry $R_{ij}$ represents user $i$'s rating of item $j$. Each row of $R$ is chosen from the set of nominal rating vectors $\{v_k\}_{k=1}^K$, depending on the cluster to which the corresponding user belongs. Specifically, if user $i$ belongs to cluster $\mathcal{C}_k$, then the $i$-th row of $R$ equals $v_k$.

\end{itemize} 

\paragraph{Observations.} Three types of observations, as illustrated in  Figure~\ref{HG}, are available: (i) a sub-sampled matrix $U$; (ii) a social graph $G$, and (iii) a collection of social hypergraphs $\{HG_d \}_{d=3}^{W}$, where $HG_d$ is a $d$-uniform hypergraph and $d$ is an integer satisfying $3 \le d \le W$. To be concrete:

1) The sub-sampled matrix $U\in \{+1,-1,*\}^{n \times m}$ is sampled from the rating matrix $R$, where the symbol $*$ represents entries that are not sampled. Specifically, each entry of the rating matrix $R$ is sampled, independently of the others, with a \emph{sample probability} $p \in [0,1]$. We also allow the presence of \emph{noise} during the sampling process, by assuming each sampled entry in $U$ may be flipped from the corresponding entry in $R$  with probability $\theta \in [0,1/2)$. 
Letting $\theta = 0$ leads to the noiseless setting. Therefore, $U_{ij}$ equals $R_{ij}$ with probability $p(1-\theta)$, equals $-R_{ij}$ with probability $p\theta$, and equals $*$ with probability $1-p$. 

2) The social graph $ G = (\mathcal{V}, \mathcal{E})$ is generated by the SBM, with $\mathcal{V} = [n]$ being the set of users and $\mathcal{E}$ being the set of edges. Let $\mathcal{E}'$ be the set that comprises all possible edges over $\mathcal{V}$, where $|\mathcal{E}'| = \binom{n}{2}$. For each $e \in \mathcal{E}'$, the pair of users connected by $e$ is denoted by $ \{ v_{e}^1, v_{e}^2 \} $, and the probability of $e$ appearing in the edge set $\mathcal{E}$ of the social graph $G$ follows the rule:
\begin{align*}
   \mathbb{P}(e \in \mathcal{E}) = \begin{cases}
   \alpha_2, \text{ \ \ \ if \ } v_{e}^1 \text{\ and  } v_{e}^2 \text{ belong to a same cluster},\\  
   \beta_2, \text{ \  \ \ otherwise}.
   \end{cases}
\end{align*}

3) Each hypergraph $HG_d = (\mathcal{V}, \mathcal{H}_d, d) $ is generated by the $d$-uniform HSBM, with $\mathcal{H}_d$ being the set of hyperedges and $d$ being the number of users in each hyperedge. Let $\mathcal{H}'_d$ be the set that comprises all possible subsets of $\mathcal{V}$ with cardinality $d$, where $|\mathcal{H}'_d| = \binom{n}{d}$. 
For each $h \in \mathcal{H}'_d$, we denote the corresponding $d$ users as $\{v_{h}^i\}_{i=1}^{d}$, and the probability of $h$ appearing in $\mathcal{H}_d$ follows the rule:
\begin{align*}
   \mathbb{P}(h \in \mathcal{H}_d) = \begin{cases}
   \alpha_d, \text{ \ \ \ if all the users in } \{v_{h}^i\}_{i=1}^d \text{ belong to a same cluster},\\  
   \beta_d, \text{ \  \ \ otherwise}.
   \end{cases}
\end{align*}


\begin{remark}
    Note that a graph can be regarded as a special $d$-uniform hypergraph with $d = 2$. Thus, we use $G$ and $HG_2$ interchangeably to represent the social graph. The aggregated graph and hypergraph information $(G, \{HG_d\}_{d=3}^W)$ can also be simplified  as $\{HG_d\}_{d=2}^W$ for brevity.
\end{remark}

\paragraph{Objectives.} Based on the sub-sampled matrix $U$, the observed graph $G$ (or $HG_2$), and the hypergraphs $\{HG_d \}_{d=3}^{W}$, the learner aims to use an estimator/algorithm $\psi = \psi(U,\{HG_d \}_{d=2}^{W})$ to achieve exact matrix completion, i.e., to exactly recover the matrix $R$ without any error. 

\section{MCH: An Efficient Matrix Completion Algorithm}\label{sec:MCH}

In this section, we introduce an efficient algorithm, named MCH, that can effectively exploit social graphs and hypergraphs to complete the rating matrix. It takes the sub-sampled rating matrix $U$, the aggregated graph and hypergraphs $\{ HG_d \}_{d=2}^{W}$ and hyperparameters $\{c_d\}_{d=2}^{W}$ as input, and outputs an estimated rating matrix $\widetilde{R} \in \{+1, -1\}^{n \times m}$ as the estimate of the ground truth matrix $R$. 

\renewcommand{\algorithmicrequire}{\textbf{Input}}
\renewcommand{\algorithmicensure}{\textbf{Output}}
\begin{algorithm}\label{alg1}
	\caption{MCH} 
	\begin{algorithmic}
		\REQUIRE Sub-sampled matrix $U$, Hypergraphs $\{HG_d \}_{d=2}^W$, Hyperparameters $\{c_d\}_{d=2}^{W}$
		\STATE \textbf{Stage 1: Partial recovery of clusters}
		\STATE Calculate the weighted adjacency matrix $A = \sum_{d=2}^W \frac{1}{d} H_d H_d^T$ based on  $\{HG_d \}_{d=2}^W$;
            \STATE Apply spectral clustering on $A$ to obtain initial estimates of clusters $\{\mathcal{C}^{(0)}_k\}_{k \in [K]}$;
		\STATE \textbf{Stage 2: Estimating rating vectors}
		\FOR{cluster $k=1$ to $K$}
		\STATE Obtain the estimated rating vector $v'_k$ via \emph{majority rule};
		\ENDFOR
		\STATE \textbf{Stage 3: Local refinements of clusters}
		\FOR{iteration $t=1$ to $T$}
		\FOR{user $i=1$ to $n$}
		\STATE $  k^* = \argmax_{k \in [K]}	 n \cdot \sum_{d=2}^{W} c_d \cdot h_d (\{i\}, \mathcal{C}_k^{(t-1)})/|\mathcal{C}_{k}^{(t-1)}|  + |\Lambda_{i}(v'_k)|$;
            \STATE Declare $i \in \mathcal{C}_{k^*}^{(t)}$;
		\ENDFOR
		\ENDFOR
		\ENSURE Estimated rating matrix $\widetilde{R}$ such that the $i$-th row equals $v'_k$ whenever user $i \in \mathcal{C}_k^{(T)}$.
	\end{algorithmic}
\end{algorithm}

\paragraph{Algorithm Description.} 
Our algorithm consists of three stages: Stage 1 aims to partially recover the user clusters using the aggregated graph and hypergraphs, Stage 2  estimates the nominal rating vectors $\{v_k\}_{k \in [K]}$ of all the clusters based on the sub-sampled matrix $U$, and Stage 3 follows an iterative procedure to refine the clusters and finally outputs an estimated matrix.   For two sets of users $\mathcal{S}_1,\mathcal{S}_2 \subseteq [n]$, we define $h_d(\mathcal{S}_1, \mathcal{S}_2)$ as the number of hyperedges that cross $\mathcal{S}_1$ and $\mathcal{S}_2$ in hypergraph $HG_d$ (i.e., the number of hyperedges that contain at least one  node in $\mathcal{S}_1$ and at least one node in $\mathcal{S}_2$).


\emph{Stage 1 (Partial recovery of clusters):}
We use the incidence matrix $H_d \in \{0, 1\}^{n \times |\mathcal{H}_d|}$ to represent the hypergraph $HG_d = (\mathcal{V}, \mathcal{H}_d, d)$, where the $(i, j)$-entry of $H_d$ equals $1$ if user $i$ belongs to the $j$-th hyperedge $h_j \in \mathcal{H}_d$, and equals $0$ otherwise. We then compute the \emph{weighted adjacency matrix} $A \triangleq \sum_{d=2}^W \frac{1}{d} H_d H_d^T$ based on $\{HG_d\}_{d=2}^W$,
where $H_d H_d^T$ is an $n \times n$ matrix with its $(i_1,i_2)$-entry representing the number of hyperedges in $HG_d$ that contain both users $i_1$ and $i_2$. We employ a \emph{spectral clustering method} (e.g., the Spectral Partition algorithm in~\cite{yun2016optimal}) on the weighted adjacency matrix $A$ to obtain an initial estimate of the $K$ clusters,  denoted by  $\{ \mathcal{C}_1^{(0)}, \mathcal{C}_2^{(0)}, \ldots, \mathcal{C}_K^{(0)}\}$. 



\emph{Stage 2 (Estimate rating vectors):} 
We estimate the nominal rating vectors based on the estimated clusters $\{ \mathcal{C}_k^{(0)}\}_{k \in [K]}$ as well as the observed ratings in the sub-sampled rating matrix $U$, based on a \emph{majority rule}. Specifically, let $\mathcal{U} \triangleq \{(i,j) \in [n] \times [m]: U_{ij} \ne * \}$ be the set of indices corresponding to the sub-sampled entries in $U$, and $v'_k \in \{+1,-1\}^m$ be the estimated rating vector of cluster $\mathcal{C}_k$. For each item $j \in [m]$, we set the value of $v'_k(j)$ to be the rating that is given by the majority of users in $\mathcal{C}_k^{(0)}$ to item $j$. Formally, $v'_k(j) = \argmax_{u\in\{+1,-1\}} \sum_{i \in \mathcal{C}_k^{(0)}} \mathbbm{1}\{U_{ij} = u\}$.

\emph{Stage 3 (Local refinement of clusters):}
In this stage we iteratively refine the user clusters using the sub-sampled rating matrix $U$, aggregated graph and hypergraphs $\{ HG_d \}_{d=2}^{W}$, and the estimated rating vectors $\{v'_k\}_{k \in [K]}$. This process operates over $T$ iterations, with each iteration building upon the output of the previous one. The outputs at the end of iteration $t$ (where $t \in [T]$) are denoted by $\{\mathcal{C}_k^{(t)}\}_{k \in [K]}$. At the $t$-th iteration, we reclassify each user $i \in [n]$ into cluster $\mathcal{C}_{k^*}^{(t)}$, where 
\begin{equation}\label{eq1}
 k^* = \argmax_{k \in [K]}	 \frac{ n \sum_{d=2}^{W} c_d \cdot h_d (\{i\}, \mathcal{C}_k^{(t-1)})}{|\mathcal{C}_{k}^{(t-1)}|}  + |\Lambda_{i}(v'_k)|,
\end{equation}
and $\Lambda_{i}(v'_k) \triangleq \{j \in [m]: U_{ij} = v'_k(j) \}$ is the set of observed ratings of user $i$ that coincide with the estimated nominal rating vector $v'_k$ of cluster $k$. After $T$ iterations, the estimated user clusters are $\{\mathcal{C}_1^{(T)}, \mathcal{C}_2^{(T)}, \ldots, \mathcal{C}_K^{(T)} \}$, and  MCH outputs the estimated rating matrix $\widetilde{R} \in \{+1,-1\}^{n \times m}$ such that the $i$-th row equals $v'_k$ whenever user $i$ belongs to $\mathcal{C}_k^{(T)}$.

\begin{remark}
In the symmetric setting to be introduced in Section~\ref{sce:MCH_Guarantee}, we provide the optimal values for the hyperparameters $\{c_d\}_{d=2}^W$, and also show that setting the number of iterations $T = O(\log n)$ is sufficient for exact recovery of the user clusters as well as the rating matrix.
\end{remark}

\paragraph{Computational Complexity.} Stage 1 runs in $O(n^2 \sum_{d=2}^W |\mathcal{H}_d|)$ time for computing the weighted adjacency matrix $A$ as well as running the spectral clustering method in~\citep{yun2016optimal}. Stage 2 runs in $O(|\mathcal{U}|)$ time, where $|\mathcal{U}|$ concentrates around $mnp$ with high probability. Stage 3 runs in $O(|\mathcal{U}|T + \sum_{d=2}^W |\mathcal{H}_d| T )$ time. Therefore, the overall complexity of MCH is $O( |\mathcal{U}|T + (n^2 + T)\sum_{d=2}^W |\mathcal{H}_d| )$.

\section{Theoretical Guarantees of MCH}\label{sce:MCH_Guarantee}
This section provides theoretical guarantees for our algorithm MCH under a specific \emph{symmetric setting}, which, on top of the model described in Section~\ref{sec:problem_state}, further requires that the $K$ disjoint clusters $\mathcal{C}_1, \mathcal{C}_2, \ldots, \mathcal{C}_K$ are of equal sizes.

\emph{\textbf{The symmetric setting}: Assume the size of each cluster $\mathcal{C}_k$ (where $k \in [K]$) satisfies $|\mathcal{C}_k| = n/K$.  }

The additional assumption is frequently applicable in practical scenarios, such as in a school setting where the $K$ clusters can be considered as $K$ classes with an equal number of students. Such a symmetric assumption has also been adopted in a number of related works in the context of matrix completion~\citep{ahn2018binary, jo2021discrete, zhang2022mc2g} and community detection~\citep{abbe2015exact, kim2018stochastic}. 
Moreover, we focus on the \emph{logarithmic average degree regime} for each hypergraph $HG_d$ where the edge generation probability $\alpha_d$ and $\beta_d$ scale as $\Theta(\log n / \binom{n-1}{d-1})$, since the gain of hypergraphs in this regime is significant and can also be precisely quantified (as demonstrated in Theorem~1 below).\footnote{The logarithmic average degree regime, where each node has an expected degree of $\Theta(\log n)$, is of particular interest in the community detection literature because the threshold for exact recovery of clusters in the hypergraph SBM falls into this regime~\citep{kim2018stochastic,zhang2022exact}.}

Before presenting the theoretical result, we first introduce the parameter $\gamma$ that quantifies the minimal \emph{Hamming distance} between pairs of nominal rating vectors $(v_{i}, v_{j})$. Formally, we have $\min_{i,j\in [K]: i \ne j}\Vert v_{i} - v_{j} \Vert_0 = \lceil \gamma m \rceil$, where $\Vert v_{i} - v_{j} \Vert_0$ is the \emph{$l_0$-norm} that counts the number of different elements in vectors $v_{i}$ and $v_{j}$, and $\lceil \gamma m \rceil$ is the smallest integer that is greater than or equal to $\gamma m$. As we shall see, the parameter $\gamma$ plays a key role in
characterizing the performance of our algorithm. We then denote the set of rating matrices that satisfy  $\min_{i,j\in [K]: i \ne j}\Vert v_i - v_j \Vert_0 = \lceil \gamma m \rceil$ by $\mathcal{R}^{(\gamma)}$.
For any estimator $\psi$, we introduce the notion of \emph{worst-case error probability} $P_{\mathrm{err}}^{(\gamma)}(\psi)$ as a \emph{metric} that measures the performance of $\psi$ when the rating matrix $R$ is from the set $\mathcal{R}^{(\gamma)}$.

\begin{definition}
{For any estimator $\psi$, the worst-case error probability with respect to $\mathcal{R}^{(\gamma)}$ is defined as 
\begin{equation}\label{eq2}
    P_{\mathrm{err}}^{(\gamma)}(\psi) \triangleq \max_{X \in \mathcal{R}^{(\gamma)}} \mathbb{P}(\psi(U, \{ HG_d \}_{d=2}^W) \neq R \ | \ R=X ),
\end{equation}
where $\mathbb{P}(\psi(U, \{ HG_d \}_{d=2}^W) \neq R \ | \ R=X )$ represents the probability (over the randomness in the generation of graph/hypergraphs, the sampling process and noise) that exact matrix completion is not achieved (i.e., the rating matrix is not exactly recovered)  when the rating matrix $R$ equals $X$.
}
\end{definition}

We are now ready to provide a sufficient condition on the sample probability $p$ that guarantees MCH to exactly recover the rating matrix $R$ (with high probability) in the symmetric setting.

\begin{theorem}\label{The1}
    Assume\footnote{The assumption that the sizes of users and items satisfy $m=\omega(\log n)$ and $m = o(e^n)$ avoids extreme cases wherein the rating matrix $R$ is excessively ``tall'' or ``fat''. This is only a mild assumption that arises from technical considerations, and is suitable for most practical scenarios.} $m=\omega(\log n)$ and $m=o(e^n)$. For any $\epsilon>0$, if sample probability $p$ satisfies
    \begin{equation}\label{eq3}
        p \geq \max \left\{ \frac{(1+\epsilon)\log n - \sum_{d=2}^{W} \frac{\binom{n-1}{d-1}}{K^{d-1}} (\sqrt{\alpha_d}-\sqrt{\beta_d})^2}{(\sqrt{1-\theta} - \sqrt{\theta})^2 \gamma m}, \frac{ (1+\epsilon) K \log m}{(\sqrt{1-\theta} - \sqrt{\theta})^2 n} \right\}, 
    \end{equation}
    then MCH ensures $\lim_{n \rightarrow \infty}P_{\mathrm{err}}^{(\gamma)} = 0$ (or equivalently, exactly recovers the rating matrix with probability approaching one), by setting $T = O(\log n)$ and $c_d =  \log \left( \frac{\alpha_d(1-\beta_d)}{\beta_d(1-\alpha_d)} \right) / \left( K \log \left( \frac{1-\theta}{\theta} \right) \right).$

\end{theorem}

\begin{proof}
    Due to the space limitation, we provide the proof in the supplementary material.
\end{proof}


While in Theorem~\ref{The1} the knowledge of the model parameters $\theta$ and $\{\alpha_d, \beta_d\}_{d=2}^{W}$ is required to determine the values of hyperparameters $\{c_d\}_{d=2}^W$, we point out that such knowledge is \emph{not necessary} since they can be estimated on-the-fly via the following expressions:
\begin{equation}\label{eq4}
    \begin{aligned}
         &\alpha'_d \triangleq \frac{ \sum_{k\in [K]} h_d(\mathcal{C}_k^{(0)}, \mathcal{C}_k^{(0)})}{ K\binom{n/K}{d}}, \quad  \beta'_d \triangleq \frac{ |\mathcal{H}_d| - \sum_{k\in [K]} h_d(\mathcal{C}_k^{(0)}, \mathcal{C}_k^{(0)})}{ \binom{n}{d} - K \binom{n/K}{d}} , \quad  \theta' \triangleq 1 -\frac{|\Lambda_{R^{(0)}}|}{|\mathcal{U}|}, 
    \end{aligned}
\end{equation}
where $R^{(0)} \in \{+1,-1\}^{n\times m}$ is the matrix such that its $i$-row equals $v'_k$ whenever $i\in \mathcal{C}_k^{(0)}$, and the set $\Lambda_{R^{(0)}} \triangleq \{ (i,j)\in \mathcal{U} : U_{ij}=(R^{(0)})_{ij} \}$ represents the collection of indices where the sub-sampled entries in $U$ coincide with the corresponding entries in $R^{(0)}$. As proved in the supplementary material, the theoretical guarantee of MCH (shown in Theorem~1) remains  valid if we replace $(\theta, \{\alpha_d\}, \{\beta_d \} )$ by $(\theta', \{\alpha'_d\}, \{\beta'_d \} )$, as long as the additional assumption $m = O(n)$ is satisfied (which means the number of items should not be much larger than the number of users).

\section{An Information-theoretic Lower Bound and The Sharp Threshold}\label{sec:lower_bound}

In this section, we provide an \emph{information-theoretic lower bound} on the sample probability $p$ for the symmetric setting (i.e., when the $K$ clusters are of equal sizes), which serves as the fundamental performance limit of \emph{any} algorithm in the considered matrix completion problem. 

\begin{theorem}\label{The2}
Assume $m=\omega(\log n)$ and $m=o(e^n)$. For any $\epsilon > 0$, if sample probability $p$ satisfies   
    \begin{equation}\label{eq5}
        p \leq \max \left\{ \frac{(1-\epsilon)\log n - \sum_{d=2}^{W} \frac{\binom{n-1}{d-1}}{K^{d-1}} (\sqrt{\alpha_d}-\sqrt{\beta_d})^2}{(\sqrt{1-\theta} - \sqrt{\theta})^2 \gamma m}, \frac{ (1-\epsilon) K \log m}{(\sqrt{1-\theta} - \sqrt{\theta})^2 n} \right\}, 
    \end{equation}
    then $\lim_{n \rightarrow \infty }P_{\mathrm{err}}^{(\gamma)}(\psi) \neq 0$ for any estimator $\psi$ under the  symmetric setting.
\end{theorem}

\begin{proof}
    Due to the space limitation, we provide the proof in the supplementary material.
\end{proof}

The information-theoretic lower bound states that any algorithm/estimator $\psi$ must fail to guarantee $\lim_{n \rightarrow \infty }P_{\mathrm{err}}^{(\gamma)}(\psi) = 0$ if the sample probability $p$ is smaller than the right-hand side of Eqn.~\eqref{eq5}, yielding a \emph{necessary condition} for exactly recovering the rating matrix. Comparing Theorems~1 and~2, we note that the sufficient condition for MCH to succeed matches the necessary condition (by letting $\epsilon \to 0$). This implies that under the symmetric setting:
\begin{enumerate}[wide, labelwidth=!, labelindent=0pt]
    \item The proposed algorithm MCH is \emph{optimal} in terms of the sample efficiency.
    \item There exists a \emph{sharp threshold} $p^*$ on the sample probability such that exact recovery of matrix $R$ is possible if and only if 
\begin{equation}\label{eq6}
    p > p^* \triangleq  \max \left\{ \frac{\log n - \sum_{d=2}^{W} \frac{\binom{n-1}{d-1}}{K^{d-1}} (\sqrt{\alpha_d}-\sqrt{\beta_d})^2}{(\sqrt{1-\theta} - \sqrt{\theta})^2 \gamma m}, \frac{ K \log m}{(\sqrt{1-\theta} - \sqrt{\theta})^2 n} \right\}.
\end{equation} 
Here, $p^*$ is referred to as the \emph{optimal sample probability} for exact recovery of the rating matrix. Below, we provide some remarks on the expression of  $p^*$.
\end{enumerate}

\begin{itemize}[wide, labelwidth=!, labelindent=0pt]
  
\item The first term of Eqn.~\eqref{eq6}, roughly speaking, is the threshold for recovering the $K$ user clusters, while the second term  is the threshold for recovering the nominal rating vectors $\{v_k\}_{k \in [K]}$. When the sample probability $p$ is greater than both terms, one can recover both the user clusters and the nominal rating vectors exactly, thus yielding the exact matrix completion of the rating matrix $R$. Otherwise, it is impossible to exactly recover either the clusters or the nominal rating vectors, leading to a failure of exact matrix completion.

\item When the noise parameter $\theta \in [0,1/2)$, the term $(\sqrt{1-\theta} - \sqrt{\theta})^{-2}$ is an increasing function of $\theta$, meaning that a larger sample probability is needed when the sampling process is noisier.

\item The optimal sample probability $p^*$ is a decreasing function of $\gamma$ (the parameter that quantifies the minimal pairwise distance between nominal rating vectors), which makes intuitive sense because a larger value of $\gamma$ means that different clusters are more separable, making it easier for recovering clusters as well as recovering the rating matrix.   
\item In addition to the optimal sample probability $p^*$, one can also define the \emph{optimal sample complexity} as  $ nmp^* = (\sqrt{1-\theta} - \sqrt{\theta})^{-2} \max \{\gamma^{-1} n (\log n - \sum_{d=2}^{W} \binom{n-1}{d-1} k^{1-d} (\sqrt{\alpha_d}-\sqrt{\beta_d})^2 ), k m \log m \}$, which corresponds to the minimum expected number of sampled entries required for achieving exact matrix completion. A direct implication is that, for the considered problem, it suffices to sample $\Theta(\max\{n\log n, m \log m\})$ matrix entries to achieve exact matrix completion.  
\end{itemize}

\paragraph{The gain of social graph and hypergraphs.} For notational convenience, we define $I_d \triangleq (\sqrt{\alpha_d}-\sqrt{\beta_d})^2$, where $2 \le d \le W$, as a measure of the ``quality'' of the hypergraph $HG_d$.\footnote{Intuitively, a small value of $I_d$ means that the difference between the probability of generating hyperedges that contain users in the same cluster and the probability of generating hyperedges that contain users in different clusters is small, making it hard to distinguish the clusters. On the contrary, the clusters are easier to be distinguished/recovered if $I_d$ is large. For HSBMs with $K$ equal-sized clusters, a recent result~\citep{zhang2022exact} states that it is possible to exactly recover the $K$ clusters when $I_d > k^{d-1} (\log n)/\binom{n-1}{d-1}$, and is impossible otherwise.} We further define  $I_h \triangleq \sum_{d=2}^{W} \binom{n-1}{d-1} k^{1-d} I_d$ as the weighted sum of the qualities of the social graph and hypergraphs. From the expression of $p^*$ in Eqn.~\eqref{eq6}, we note that:

\begin{itemize}[wide, labelwidth=!, labelindent=0pt]
    \item When $I_h = o(\log n)$, the contribution of exploiting the social graph and hypergraphs is negligible.
    \item When $I_h = \Omega(\log n)$ and $I_h < \log n -k \gamma n^{-1} m \log m$, exploiting the social graph and hypergraphs helps to reduce the optimal sample probability $p^*$ by $I_h(\sqrt{1-\theta} - \sqrt{\theta})^{-2}(\gamma m)^{-1}$. When $I_h \ge \log n -k \gamma n^{-1} m \log m$, the gain due to the graph and hypergraphs \emph{saturates} (i.e., the maximum gain is achieved), since the second term in Eqn.~\eqref{eq6} becomes the dominant term.\footnote{Recall that the second term in Eqn.~\eqref{eq6} represents the minimal number of samples required for recovering the nominal rating vectors, thus increasing the quality of graph or hypergraphs will be no longer helpful.} Thus, the maximum gain due to the social graph and hypergraphs is $g^* \triangleq (\sqrt{1-\theta} - \sqrt{\theta})^{-2} (\frac{\log n}{\gamma m} - \frac{k \log m}{n})$.
\end{itemize}

\begin{figure}
    \centering

    \begin{subfigure}[b]{0.48\linewidth}
        \centering
        \includegraphics[width=0.8\linewidth]{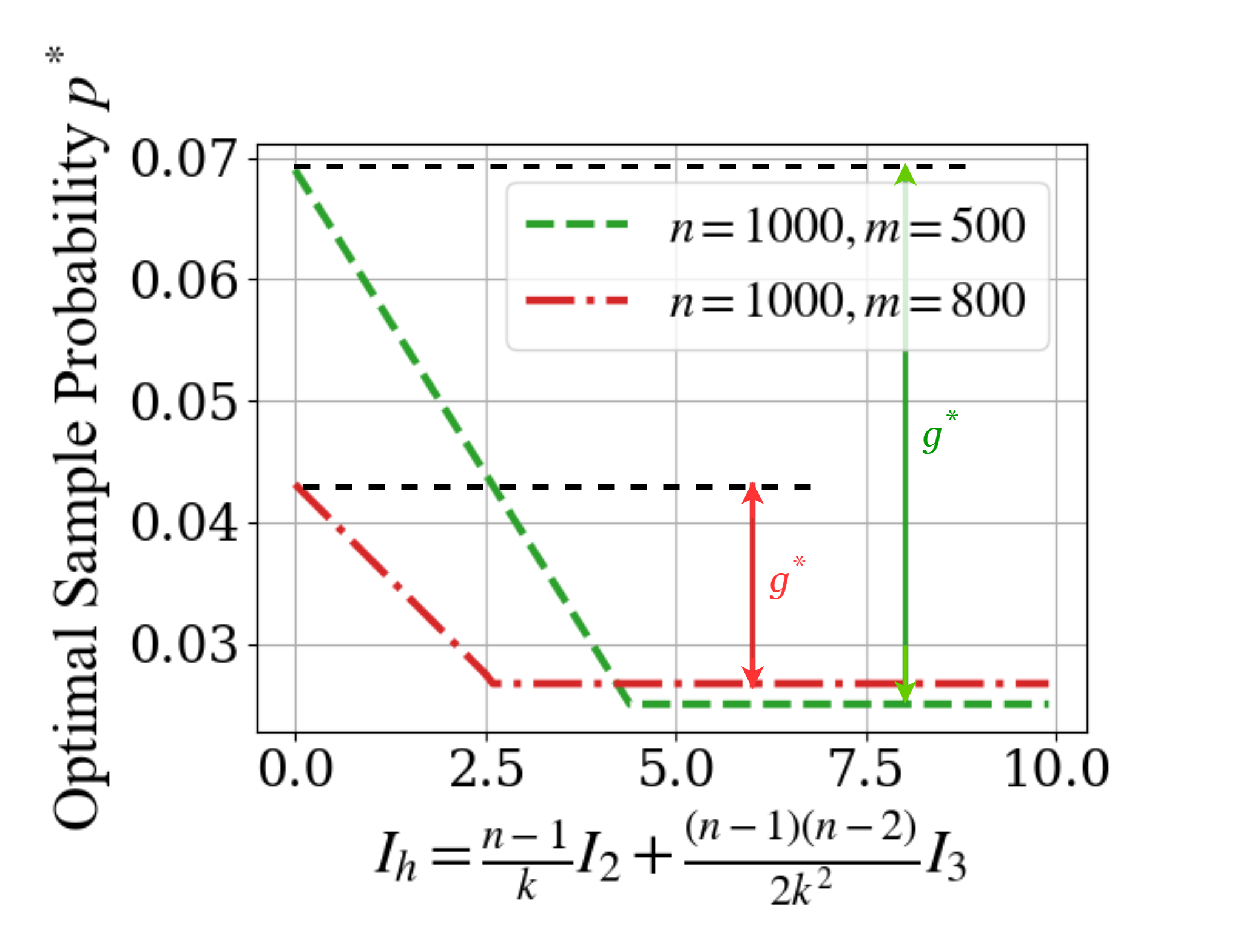}
        \caption{Optimal sample probability $p^*$ versus $I_h$.}
        \label{T1}
    \end{subfigure}
    \hfill
    \begin{subfigure}[b]{0.48\linewidth}
        \centering\includegraphics[width=0.77\linewidth]{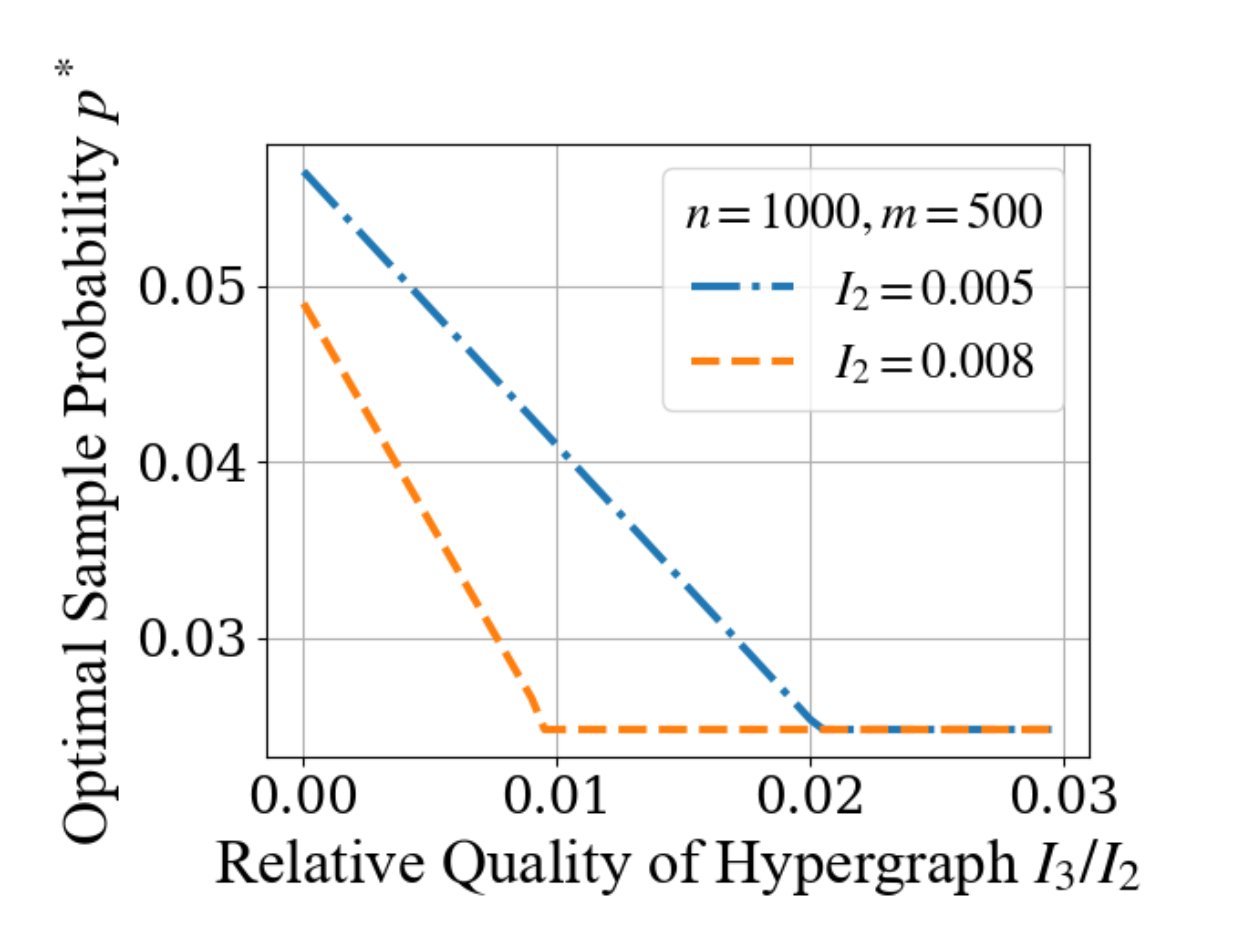}
        \caption{Optimal sample probability $p^*$ versus the ``relative quality'' of hypergraphs (measured by $I_3/I_2$).}
        \label{T2}
    
    \end{subfigure}

    \caption{\small Consider a setting that contains $K = 4$ clusters, a social graph $HG_2$, and a $3$-uniform hypergraph $HG_3$. Let $\gamma=0.2$ and $\theta=0$. Figure~\ref{T1} visualizes the gain due to $HG_2$ and $HG_3$ in terms of reducing the optimal sample probability $p^*$, where  $g^*$ represents the maximum possible gain. Figure~\ref{T2} shows the extra gain due to exploiting the hypergraph $HG_3$ for fixed values of the graph quality $I_2$. Note that $I_3/I_2=0$ means that no hypergraph information is available, corresponding to the setting considered in \citep{ahn2018binary}.}
    \label{F1}
    \vspace{-8pt}
\end{figure}

In Figure~\ref{T1}, we illustrate the amount of reduction in the optimal sample probability $p^*$ for different values of $I_h$, under a setting that contains $K=4$ equal-sized clusters, a social graph $HG_2$, and a $3$-uniform hypergraph $HG_3$. It is clear from Figure~\ref{T1} that, for both the parameter settings $(n,m) = (1000,500)$ and $(n,m) = (1000,800)$,  the optimal sample probability $p^*$ first decreases linearly with $I_h$, and then stays constant after $I_h$ exceeding $\log n -k \gamma n^{-1} m \log m$.   Comparing the red and green lines in Figure~\ref{T1}, we note that a larger relative value of $n$ results in a larger maximum gain due to social graph and hypergraphs (which is represented by $g^*$ in the figure).

\paragraph{The additional gain of exploiting hypergraphs.} Compared to the prior work~\citep{ahn2018binary} that only utilizes graph information for matrix completion, our theoretical results show that exploiting additional social hypergraphs leads to an \emph{extra gain} of $\sum_{d=3}^{W} \binom{n-1}{d-1} k^{1-d} I_d(\sqrt{1-\theta}-\sqrt{\theta})^{-2} (\gamma m)^{-1}$ in terms of reducing the optimal sample probability. This gain becomes more significant as the ``relative  quality'' of hypergraphs (over the quality of the graph) improves. In Figure~\ref{T2}, we plot the optimal sample probability $p^*$ as a function of the relative quality of hypergraphs (measured by the ratio of $I_3$ to $I_2$), assuming there is only a single hypergraph $HG_3$, and the value of graph quality $I_2$ is fixed.

\section{Experimental Results}\label{sec:experiments}

\begin{figure*}[ht]
	\centering
	\subfloat[\small Synthetic data ($n= 3m$)]{\includegraphics[width=0.3\linewidth]{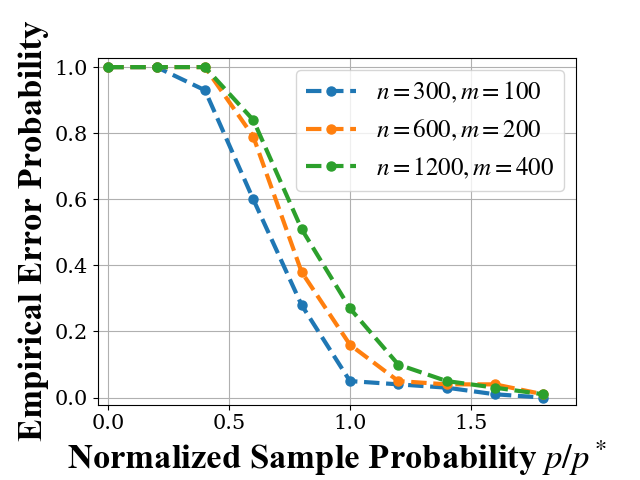}%
		\label{E1}}
	\hfill
	\subfloat[\small Synthetic data ($m= 3n$)]{\includegraphics[width=0.3\linewidth]{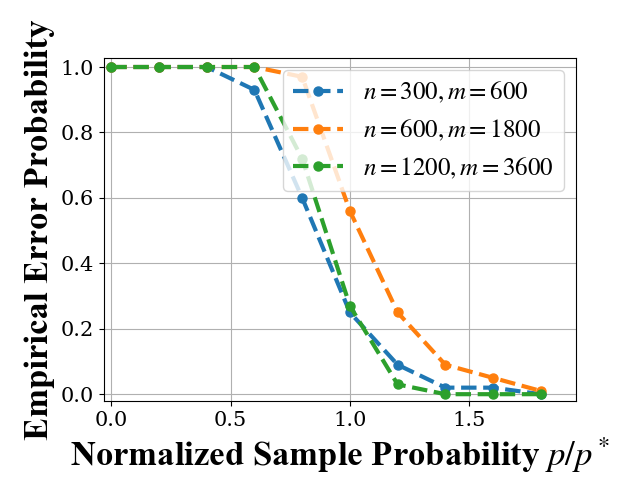}%
		\label{E2}}
        \hfill
	\subfloat[\small Synthetic data (different $\hat{I}_3$)]       
        {\includegraphics[width=0.3\linewidth]{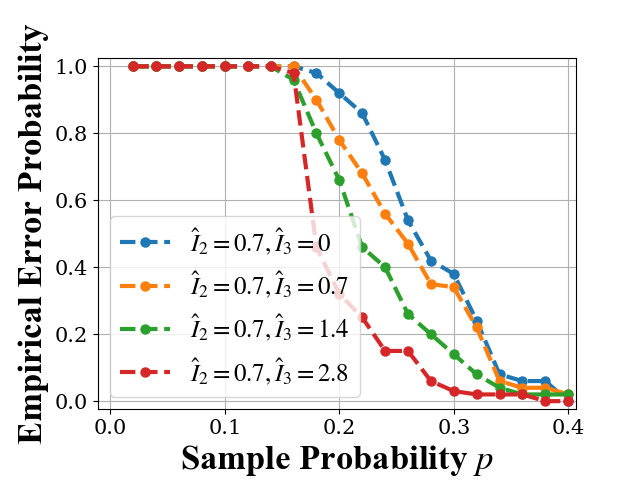}%
		\label{E3}}
        \hfill
	\subfloat[\small Semi-real data (performance comparison)]      
        {\includegraphics[width=0.35\linewidth]{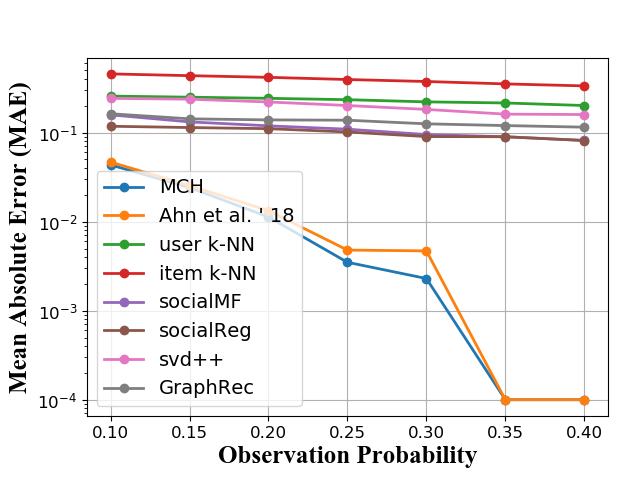}%
		\label{R1}}
        \hspace{10pt}
	\subfloat[\small Semi-real data from modified social networks (fix $p = 0.1$)]{\includegraphics[width=0.35\linewidth]{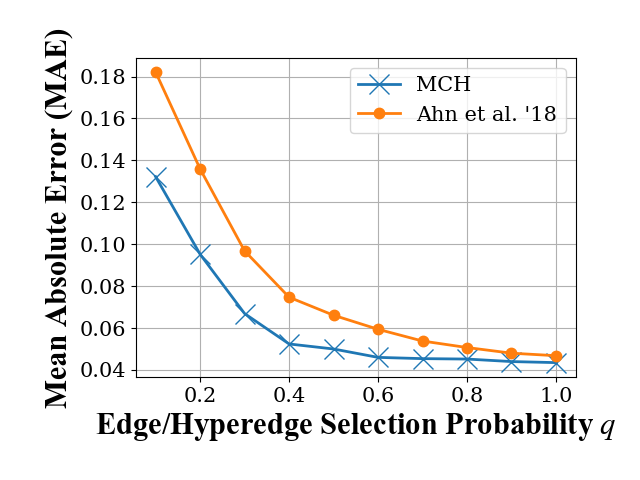}%
		\label{R2}}
	\caption{\small Experimental results on synthetic and semi-real datasets show  the superior performance of MCH.}
	\label{fig_sim}
\end{figure*}

\paragraph{Experiments on synthetic datasets.} We first conduct experiments on synthetic datasets (generated according to the model in Section~\ref{sec:problem_state}) to validate the theoretical guarantee of MCH provided in Theorem~\ref{The1}. In Figures~\ref{E1} and~\ref{E2}, we consider a setting that contains $K=3$ equal-sized clusters, a graph of quality $I_2 = \log n/n$, and a $3$-uniform hypergraph of quality $I_3 =2 \log n/\binom{n-1}{2}$. We set the noise parameter $\theta = 0.1$ and $\gamma=0.4$. We plot the \emph{empirical error probability} (defined as the fraction of the trials where exact matrix completion is not achieved out of $100$ trials) as a function of the \emph{normalized sample probability} (defined as the ratio of the sample probability $p$ to the optimal sample probability $p^*$).  It is clear that the empirical error probability tends to zero when the normalized sample probability exceeds one (i.e., when the sample probability exceeds $p^*$), and is bounded away from zero otherwise.  This indicates a strong agreement with our theory. 

In Figure~\ref{E3}, we consider a different synthetic dataset with $n=300$, $m=100$, $I_2=\hat{I}_2 \log n/ n$ and $I_3 = \hat{I}_3 \log n/\binom{n-1}{2}$ for multiple different values of $\hat{I}_3$, in order to examine the extra gain due to hypergraphs with different qualities. Comparing the four lines in Figure~\ref{E3}, it is evident that utilizing hypergraph information helps to reduce the error probability, and the amount of reduction becomes more significant as the quality of the hypergraph improves.


\paragraph{Experiments on a semi-real dataset.} We also evaluate the performance of MCH on a semi-real dataset that consists of a real social network where the interactions between users are captured by both graph and  hypergraphs. The social network, named \emph{contact-high-school dataset} \citep{chodrow2021hypergraph, Mastrandrea-2015-contact}, comprises of $327$ student users that belong to $9$ disjoint classes, with the size of each class ranging from $29$ to $44$. It contains $5,498$ ordinary edges and $2,320$ hyperedges, where the size of each hyperedge ranging from $3$ to $5$.
Building upon the contact-high-school social network, we then synthesize a rating matrix with $m = 90$ items and $9$ nominal rating vectors with minimal fractional Hamming distance $\gamma = 0.22$. We set the noise parameter $\theta = 0.1$ in the sampling process. 

In Figure~\ref{R1}, we compare\footnote{The values of hyperparameters $\{c_d \}_{d=2}^5$ in our algorithm MCH are all set to $0.01$.} MCH with several representative matrix completion algorithms, including  user k-NN, item k-NN, svd++~\cite{koren2008factorization}, SocialMF~\cite{jamali2010matrix}, SocialReg~\cite{ma2011recommender}, GraphRec\footnote{The GraphRec employs a batch size of $64$, conducts training for $10$ epochs, adopts a learning rate of $0.001$, and employs an embedding dimension of $32$.}~\cite{fan2019graph}, and the spectral clustering-based algorithm that only utilizes graphs (by Ahn et al.~\cite{ahn2018binary}). The performance is measured by the \emph{mean absolute error} (MAE) defined as $\sum_{i \in [n], j \in [m]} \mathbbm{1}\{\widetilde{R}_{ij} \ne R_{ij} \}/(mn)$. Figure~\ref{R1} shows that \emph{MCH outperforms all the competitors}. Note that the performance of the algorithm by Ahn et al.~\cite{ahn2018binary} approaches to ours, which is because the quality of the social graph in this real dataset is already high enough, so utilizing the graph information only (without the hypergraphs) also results in a good performance. To further demonstrate the superiority of MCH over the one by Ahn et al.~\cite{ahn2018binary}, we consider  modified contact-high-school datasets where each edge/hyperedge in the original dataset is selected (resp. discarded) with probability $q$ (resp. $1-q$). As depicted in Figure~\ref{R2}, as  $q$ decreases (i.e., as the quality of the graph decreases), the advantage of MCH becomes more pronounced.

\section*{Acknowledgments}
This work was supported by the National Natural Science Foundation of China (Nos. U22B2036, 11931015), the Fundamental Research Funds for the Central Universities (Nos. G2024WD0151, D5000240309) and the Tencent Foundation and XPLORER PRIZE. 

\medskip

\bibliographystyle{unsrtnat}
\bibliography{reference}

\appendix

\section{Appendix}

\textbf{Outline:} This technical appendix is organized as follows. In Section~\ref{section1}, we list some assumptions that are adopted in the proofs, and introduce the notations. Section~\ref{section2} provides some technical results that are useful in the subsequent analyses. In Section~\ref{section3}, we present a detailed proof of the theoretical guarantee of the proposed algorithm MCH (Theorem~1 in the main paper). Section~\ref{section4} presents the detailed proof of the information-theoretic lower bound (Theorem~2 in the main paper). The proofs of several lemmas are deferred to Section~\ref{section5}. Additional experiments are provided and explained in detail in Section~\ref{experiments}. Section~\ref{compare} presents a detailed comparison between this study and~\cite{ahn2018binary}.


\section{Preliminaries}\label{section1}
\
\paragraph{List of Underlying Assumptions.}
The proofs of Theorem~1 and Theorem~2 rely on several assumptions on the model parameters $(n, m, K, \theta, \gamma, \{ \alpha_d \}, \{ \beta_d \})$. We list them before proceeding with the formal proofs.
\begin{itemize}[wide, labelwidth=!, labelindent=10pt]
    \item Assume $m=\omega(\log n)$ and $m = o(e^n)$. This assumption avoids extreme cases wherein the rating matrix $R$ is excessively ``tall'' or ``fat''. This is only a mild assumption that arises from technical considerations, and is suitable for most practical scenarios.

    \item When the model parameters $(\theta, \{\alpha_d \}, \{\beta_d \})$ are not known \emph{a priori}, we further assume $m=O(n)$, so that we can reliably estimate $(\theta, \{\alpha_d\}, \{\beta_d \} )$ in the proposed algorithm MCH. If these parameters are known a priori, this assumption can be discarded.
    
    \item The parameters $K, \gamma$ and $\theta$ all scale as constants that do not grow with $n$ or $m$.

    \item For each hypergraph $HG_d$, we assume $\alpha_d > \beta_d$,  which reflects most practical scenarios in which users belonging to a same  cluster are more likely to be connected than users belonging to different clusters. Moreover, we assume $\alpha_d, \beta_d = \Theta((\log n)/\binom{n-1}{d-1})$ such that the average degree of each node scales as $\Theta(\log n)$. This corresponds to the \emph{logarithmic average degree regime} that is of particular interest in the community detection literature, since the threshold for exact recovery of clusters in the HSBM falls into this regime~\citep{kim2018stochastic,zhang2022exact}.
    
\end{itemize}

\paragraph{Notations and Abbreviations.} 
For any random variable $Z$, let  $\mathbb{M}_Z (t)$ be the \emph{moment-generating function} of $Z$. For any two sets $\mathcal{A}$ and $\mathcal{B}$, we use $\mathcal{A} \triangle \mathcal{B}$ to denote the \emph{symmetric difference} of the two sets, i.e., $ \mathcal{A} \triangle \mathcal{B} \triangleq (\mathcal{A} \setminus \mathcal{B}) \cup (\mathcal{B} \setminus \mathcal{A})$.

For notational convenience, we define $$a_d \triangleq \ad, \ \ b \triangleq \bb, \ \ I_d \triangleq (\sqrt{\alpha_d} - \sqrt{\beta_d})^2, \ \ I_{\theta} \triangleq (\sqrt{1-\theta} - \sqrt{\theta})^2.$$ 
These abbreviations are frequently used throughout the supplemental material.

\section{Maximum-likelihood Function and Large Deviations Bounds}\label{section2}
\paragraph{Maximum-likelihood Function:} We denote the ground truth rating matrix as $R \in \{+1,-1\}^{n \times m}$, the $K$ user clusters  as 
$\{ \mathcal{C}_k \}_{k \in [K]}$, and the corresponding nominal rating vectors as $\{ v_k \}_{k \in [K]}$. The observations include the collection of hypergraphs $\{HG_d \}_{d=2}^W$ as well as the sub-sampled rating matrix $U \in \{+1,-1, *\}^{n \times m}$. Given the observations, we first provide the expression of the \emph{log-likelihood function} for each matrix $X$ in Lemma~1 below. 

Recall that we use $\mathcal{R}^{(\gamma)}$ to denote the set of rating matrices that satisfy $\min_{i,j\in[K]: i \ne j} \Vert v_i - v_j\Vert_0 = \lceil \gamma m \rceil$.  For two sets of users $\mathcal{S}_1,\mathcal{S}_2 \subseteq [n]$, we define $h_d(\mathcal{S}_1, \mathcal{S}_2)$ as the number of hyperedges in which all the constituting users belong to $\mathcal{S}_1 \cup \mathcal{S}_2$, and at least one user belongs to $\mathcal{S}_1$ and at least  one user belongs to  $\mathcal{S}_2$.  
For any matrix $X \in \mathcal{R}^{(\gamma)}$, let $\{\mathcal{C}_k^X\}_{k \in [K]}$ be the $K$ clusters associated with matrix $X$. Note that $\sum_{k=1}^K h_d( \mathcal{C}^X_k,  \mathcal{C}^X_k)$ is the number of \emph{in-cluster hyperedges} (i.e., the hyperedges that contain users belonging to the same cluster) with respect to $\{\mathcal{C}_k^X\}_{k \in [K]}$  in the hypergraph $HG_d$.
\begin{lemma}\label{lemma1}
    The log-likelihood function of matrix $X$, denoted as $L(X)$, is given as
    \begin{equation}
        L(X) = \sum_{d=2}^{W} a_d \cdot \sum_{k=1}^K h_d( \mathcal{C}^X_k,  \mathcal{C}^X_k) + b |\Lambda_X| + C,
    \end{equation}
    where $\{a_d\}_{d=2}^W$ and $b$ are defined in Section~1, the set $\Lambda_X\triangleq \{ (i, j) \in \mathcal{U} :  U_{ij}=X_{ij} \}$, and $C$ is a constant that is independent of the choice of $X$.
\end{lemma}
\begin{proof}
    See Section~\ref{section5} for the detailed proof.
\end{proof}

\paragraph{Large deviations bounds:} Below, we provide two large deviations results (in Lemmas~\ref{lemma2} and \ref{lemma3}) that are crucial for the subsequent proofs.  Let $\{K_d\}_{d=2}^W$ and $L$ be positive integers, and we further introduce a set of random variables that will play a role in the analyses: 
$$\{A_{dj} \}_{j=1}^{K_d} \overset{\text{i.i.d}}{\sim} \text{Bern}(\alpha_d), \ \ \{B_{dj}\}_{j =1}^{K_d} \overset{\text{i.i.d}}{\sim} \text{Bern}(\beta_d), \ \ \{P_{i}\}_{i=1}^L \overset{\text{i.i.d}}{\sim} \text{Bern}(p) ,\ \  \{\Theta_i\}_{i=1}^L \overset{\text{i.i.d}}{\sim} \text{Bern}(\theta).$$


\begin{lemma}\label{lemma2}
For any $y \geq 0$,
\begin{equation}
\begin{aligned}
    &\mathbb{P} \Big( \sum_{d=2}^{W} a_d  \sum_{j=1}^{K_d} (B_{dj} - A_{dj}) + b \sum_{i=1}^{L}P_i(2\Theta_i - 1) \geq -y  \Big) \\
    &\qquad\qquad\qquad\leq \exp \Big\{-\frac{1}{2}y - \sum_{d=2}^{W}(1+o(1))K_d I_d - (1+o(1)) L p I_\theta \Big\}. 
    \end{aligned}
    \end{equation}
\end{lemma}
\begin{proof}
    The proof relies on the Chernoff bound, and is deferred to Section~\ref{section5}.
\end{proof}

\begin{lemma}\label{lemma3}
Assuming that $\max \left\{ \sqrt{\alpha_d \beta_d}K_d, pL \right\}=\omega(1)$. Then
\begin{equation}\label{eqn4}
    \begin{aligned}
&\mathbb{P} \left( \sum_{d=2}^{W} a_d  \sum_{j=1}^{K_d} (B_{dj} - A_{dj}) + b \sum_{i=1}^{L}P_i(2\Theta_i - 1) \geq 0  \right) \\
 &\qquad\qquad\qquad \geq \frac{1}{4}\exp \left\{ - \sum_{d=2}^{W}(1+o(1))K_d I_d - (1+o(1)) L p I_\theta \right\}.
    \end{aligned}
\end{equation}
\end{lemma}
\begin{proof}
    The proof is deferred to Section~\ref{section5}.
\end{proof}

\section{Proof of Theorem~1}\label{section3}
In this section, we prove that, by setting the number of iterations $T=O(\log n)$, MCH ensures the worst-case error probability $P_{\text{err}}^{(\gamma)}$ tends to zero as long as the sample probability $p$  satisfies
\begin{align}
p \geq \max \left\{ \frac{(1+\epsilon)\log n - \sum_{d=2}^{W} \frac{\binom{n-1}{d-1}}{K^{d-1}} (\sqrt{\alpha_d}-\sqrt{\beta_d})^2}{(\sqrt{1-\theta} - \sqrt{\theta})^2 \gamma m}, \frac{ (1+\epsilon) K \log m}{(\sqrt{1-\theta} - \sqrt{\theta})^2 n} \right\}. \label{eq:thm1}
\end{align}
Using the abbreviations $I_d$ and $I_{\theta}$, the condition in~\eqref{eq:thm1} is equivalent to the following:
\begin{equation}
    \sum_{d=2}^W \frac{\binom{n-1}{d-1}}{K^{d-1}} I_d + \gamma m p I_{\theta} \geq (1+\epsilon) \log n \text{ \ and \ } \frac{1}{K} n p I_{\theta} \geq (1+\epsilon) \log m.
\end{equation}

\subsection{Analysis of Stage 1: Partial Recovery of Clusters}

First note that for each hypergraph $HG_d$, by assumption, the average degree of each node in $HG_d$ scales as $\Theta(\log n)$. Applying the \emph{Spectral Partition algorithm}~\cite{yun2016optimal}  to the weighted adjacency matrix $A$, and by a simple generalization of the proof techniques in~\cite{yun2016optimal} (for the SBM) and~\cite{zhang2022exact} (for the HSBM), one can show that when the ``quality'' of each  graph/hypergraph satisfies $I_d  = \omega(1/n^{d-1})$, the estimated clusters $\{ \mathcal{C}_1^{(0)}, \ldots, \mathcal{C}_K^{(0)} \}$ coincide with the true clusters $\{ \mathcal{C}_1, \ldots, \mathcal{C}_K \}$ except for a vanishing fraction of nodes. Formally, we define  $$\eta_k \triangleq \frac{|\mathcal{C}_k^{(0)} \setminus \mathcal{C}_k|}{n}$$ 
as the fraction of nodes that are misclassified to $\mathcal{C}_k^{(0)}$, and we have that with probability $1- o(1)$, $\eta_k = o(1)$ for all $k \in [K]$.


\subsection{Analysis of Stage 2: Exact recovery of Rating Vectors}
We now estimate the probability of failing to exactly recover the nominal rating vector $v_k$ (for each $k \in [K]$). First of all, we consider each item $j\in [m]$ separately, and  calculate the probability $\mathbb{P} (v'_k(j) \neq v_k(j))$ corresponding to the event that the estimated rating $v'_k(j)$ is not equal to the ground truth rating $v_k(j)$. Without loss of generality, we assume $v_k(j)=+1$, and by the estimation rule of Stage~2, we have
\begin{align}
   \mathbb{P} (v'_k(j) \neq v_k(j)) =  \mathbb{P} \left( \sum_{i \in \mathcal{C}^{(0)}_k  } U_{ij} \leq 0 \right) & = \mathbb{P} \left( \sum_{i \in \mathcal{C}^{(0)}_k \setminus \mathcal{C}_k } U_{ij} + \sum_{i \in \mathcal{C}^{(0)}_k \cap \mathcal{C}_k } U_{ij} \leq 0 \right) \\
    & \leq \mathbb{P} \left( \sum_{i=1}^{(\frac{1}{K} - \eta_k)n} P_i (1-2\Theta_i) - \sum_{i=1}^{\eta_k n} P'_i \leq 0 \right) \label{eqn8} \\
    & = \mathbb{P} \left( \sum_{i=1}^{(\frac{1}{K} - \eta_k)n} P_i (2\Theta_i - 1)  \geq -\sum_{i=1}^{\eta_k n} P'_i \right), \label{eqn9}
\end{align}
where $\{P_i\} \overset{\text{i.i.d}}{\sim} \text{Bern}(p)$, $\{\Theta_i\} \overset{\text{i.i.d}}{\sim} \text{Bern}(\theta) $, and $\{P'_i\} \overset{\text{i.i.d}}{\sim} \text{Bern}(p)$.  With a slight abuse of notations, we treat  $U_{ij} = *$ as $U_{ij} = 0$ when calculating $\sum_{i \in \mathcal{C}^{(0)}_k } U_{ij}$.   Eqn.~\eqref{eqn8} follows from the fact that $\sum_{i \in \mathcal{C}^{(0)}_k \setminus \mathcal{C}_k } U_{ij} \geq -\sum_{i=1}^{\eta_k n} P'_i$. The following Lemma gives a large deviation result of $\sum_{i=1}^{\eta_k n} P_i'$.

\begin{lemma}\label{lemma4}
	Suppose $Y \sim \text{Binom}(\tau n, p) $ where $0 < \tau < 1$ and $0<p<\frac{1}{2}$. Then for any $c > 2e$, 
	$$
	\mathbb{P}\left( Y \geq \frac{cnp}{\log \frac{1}{\tau}} \right) \leq 2 \exp \left( -\frac{cnp}{2} \right).
	$$ 
\end{lemma}

\begin{proof}
    See the proof in Section~\ref{section5}.
\end{proof}

According to Lemma~\ref{lemma4} and the fact $np=\Omega(\log m)$,  we have $\mathbb{P}\left( \sum_{i=1}^{\eta_k n} P_i' \geq \frac{cnp}{\log \frac{1}{\eta_k}} \right) \leq 2 \exp \left( -\frac{cnp}{2} \right)  = o(m^{-1})$. Then, Eqn.~\eqref{eqn9} is upper-bounded by
\begin{align*}
    & \mathbb{P} \left( \sum_{i=1}^{(\frac{1}{K} - \eta_k)n} P_i(2\Theta_i - 1) \geq -\frac{cnp}{\log \frac{1}{\eta_k}} \right) \cdot \mathbb{P}\left( \sum_{i=1}^{\eta_k n} P_i' \leq \frac{cnp}{\log \frac{1}{\eta_k}} \right) \\
    & + \mathbb{P} \left( \sum_{i=1}^{(\frac{1}{K} - \eta_k)n} P_i(2\Theta_i - 1) \geq -\sum_{i=1}^{\eta_k n} P_i' \right) \cdot \mathbb{P}\left( \sum_{i=1}^{\eta_k n} P_i' \geq \frac{cnp}{\log \frac{1}{\eta_k}} \right) \\
    & \leq \mathbb{P} \left( \sum_{i=1}^{(\frac{1}{K} - \eta_k)n} P_i(2\Theta_i - 1) \geq -\frac{cnp}{\log \frac{1}{\eta_k}} \right) + o(m^{-1}) \\
    & \overset{\text{(i)}}{\leq} \exp \left( \frac{1}{2} \log \left( \frac{1-\theta}{\theta} \right) \frac{c}{\log \frac{1}{\eta_k}} n p - (1+o(1)) \left( \frac{1}{K} - \eta_k  \right) n p I_\theta \right) + o(m^{-1}) \\
    & \overset{\text{(ii)}}{=} \exp \left( - (1+o(1)) \left( \frac{1}{K} - \eta_k \right) n p I_\theta + o(n p I_\theta) \right) + o(m^{-1}) \\
    & \overset{\text{(iii)}}{\leq} \exp \left( -(1+ \epsilon/2) \log m \right) +o(m^{-1}) = o(m^{-1}).
\end{align*}

Here, (i) follows from Lemma~\ref{lemma2} with $\{K_d\}_{d=2}^W = 0$, $L=(1/K - \eta_k)n$ and $y = -\frac{cnp}{\log \frac{1}{\eta_k}}$; (ii) holds since $np=\Theta(n I_\theta)$ and $\frac{c}{\log \frac{1}{\eta_k}} = o(1)$; (iii) is true due to $(1/K-\eta_k)n p I_\theta \geq (1+\epsilon/2)\log m$, which can be derived from the facts $(n p I_\theta)/K \geq (1+\epsilon) \log m$ and $\lim_{n \rightarrow \infty} \eta_k = 0$. Then, taking a union bound over all the $m$ items  and $K$ nominal rating vectors, we have $$\mathbb{P}(\exists k \in [K] \text{\ such that \ } v'_k \neq v_k ) = o(1).$$

\subsection{Analysis of Stage 3: Exact Recovery of Clusters}

As the nominal rating vectors $\{ v_k \}_{k \in [K]}$ can be exactly recovered with high probability after Stage~2, when analyzing Stage~3, we assume without loss of generality that the knowledge of $\{ v_k \}_{k \in [K]}$ is given.
According to Lemma~\ref{lemma1}, we obtain the \emph{local} log-likelihood function of user $i \in [n]$ belonging to cluster $\mathcal{C}_k$ as follows:
\begin{equation}\label{T1S31}
    L(i; \mathcal{C}_k) \triangleq \sum_{d=2}^{W} a_d \cdot h_d(\{i\}, \mathcal{C}_k) + b |\Lambda_i(v_k)| + C,
\end{equation}
where $C$ is a constant that is independent of $k$.
At the $t$-th iteration, the local refinement rule of MCH is to reclassify each user $i \in [n]$ to cluster $\mathcal{C}_{k^*}^{(t)}$, where 
\begin{equation}\label{T1S32}
    k^* =  \argmax_{k \in [K]}	  \sum_{d=2}^{W} \frac{a_d}{b} \cdot h_d (\{i\}, \mathcal{C}_k^{(t-1)})  + |\Lambda_{i}(v_k)|.
\end{equation}
Thus, the refinement rule in Eqn.~\eqref{T1S32} can be viewed as an approximation of the local log-likelihood function in Eqn.~\eqref{T1S31}, with each $\mathcal{C}^{(t-1)}_k$ being the estimate of the true cluster $\mathcal{C}_k$. Below, we introduce a property of the  local log-likelihood function that is crucial for our analysis. 
\begin{lemma}\label{lemma5}
    For any user $i$, assume $i$ belongs to cluster $\mathcal{C}_a$ for some $a \in [K]$. If $\sum_{d=2}^W \frac{\binom{n-1}{d-1}}{K^{d-1}} I_d + \gamma m p I_{\theta} \geq (1+\epsilon) \log n$, then there exists a small constant $\tau > 0$ such that the following statement holds with probability $1-O(n^{-\epsilon / 2})$:
    \begin{equation}\label{T1S33}
            L(i; \mathcal{C}_a) > \max_{\Bar{a} \in [K]: \Bar{a} \neq a} L(i; \mathcal{C}_{\Bar{a}}) + \tau \log n, \quad  \text{for all the users \ } i \in [n]. 
    \end{equation}
\end{lemma}

\begin{proof}
    See Section~\ref{section5} for the detailed proof.
\end{proof}

We denote $L(i; \mathcal{C}_a, \mathcal{C}_{\Bar{a}}) \triangleq L(i; \mathcal{C}_a)-L(i; \mathcal{C}_{\Bar{a}})$, where $a$ is the ground truth cluster that user $i$ belongs to. Then, for any $\Bar{a} \neq a$,  Eqn.~\eqref{T1S33} is equivalent to 
\begin{equation}\label{T1S34}
    L(i; \mathcal{C}_a, \mathcal{C}_{\Bar{a}}) \geq \tau \log n.
\end{equation}

Let $\mathcal{Z}_\delta$ be the set of partitions $\{ \mathcal{C}_k^z \}_{k \in [K]}$ of the $n$ user nodes, which satisfy (i) $\mathcal{C}_{k_1}^z \cap \mathcal{C}_{k_2}^z = \emptyset$ for any $k_1, k_2 \in [K]$, (ii) $\cup_{k \in [K]}\mathcal{C}_k^z = [n]$, (iii) $\sum_{k \in [K]} \mathcal{C}_k^z \setminus \mathcal{C}_k = \delta n$, where $\delta \in [ 1/n, 1/2 )$. It suffices to prove that there exists a constant $\delta'>0$ such that if $\delta < \delta'$, the following event happens with high probability: 
\begin{itemize}
    \item \emph{For any partition $\{\mathcal{C}_k^z\}_{k \in [K]} \in \mathcal{Z}_{\delta}$, the output after a single iteration belongs to $\mathcal{Z}_{\delta / 2}$.}
\end{itemize}
Then, running $T=\frac{\log (\delta' n)}{\log 2}=O(\log n)$ iterations ensures exact recovery of clusters. The formal statement  is given in the following lemma.

\begin{lemma}\label{lemma6}
    For any constant $\tau > 0$, there exists $\delta' < 1/2$ such that if $\delta < \delta'$, the following statement holds with probability $1-O(n^{-1})$: 
    for any partition $\{ \mathcal{C}_k^z \}_{k \in [K]} \in \mathcal{Z}_\delta$ and for any $\Bar{a} \neq a$,
    \begin{equation}\label{T1S35}
        |L(i; \mathcal{C}_a^z, \mathcal{C}_{\Bar{a}}^z)-L(i; \mathcal{C}_a, \mathcal{C}_{\Bar{a}})| \leq \frac{\tau}{2} \log n,
    \end{equation}
    for all except $\delta n/2$ nodes.
\end{lemma}
\begin{proof}
    The detailed proof is delivered in Section~\ref{section5}.
\end{proof}

Note that if node $i$ satisfies Eqn.~\eqref{T1S35}, by Eqn.~\eqref{T1S34} and the triangle inequality, we have $L(i; \mathcal{C}_a^z, \mathcal{C}_{\Bar{a}}^z) \ge \frac{\tau n}{2}$, meaning that node  $i$ will be classified into the true cluster based on the refinement rule. According to Lemma~\ref{lemma6}, the number of nodes that do not satisfy Eqn.~\eqref{T1S35} (which are likely to be misclassified) will be reduced by half after one iteration. Thus, running $T=O(\log n)$ iterations yields exact recovery of clusters. 

The aforementioned proof assumes the parameters $\{\af\}_{d=2}^W, \{\bt\}_{d=2}^W$, and $\theta$ are known \emph{a priori}. When these parameters are not known, one can estimate them using the following rule:
$$
\alpha'_d \triangleq \frac{ \sum_{k\in [K]} h_d(\mathcal{C}_k^{(0)}, \mathcal{C}_k^{(0)})}{ K\binom{n/K}{d}}, \quad  \beta'_d \triangleq \frac{ |\mathcal{H}_d| - \sum_{k\in [K]} h_d(\mathcal{C}_k^{(0)}, \mathcal{C}_k^{(0)})}{ \binom{n}{d} - K \binom{n/K}{d}} , \quad  \theta' \triangleq 1 -\frac{|\Lambda_{R^{(0)}}|}{|\mathcal{U}|}.
$$
For simplicity, we define $a'_d \triangleq \log \left( \frac{\alpha'_d (1-\beta'_d)}{\beta'_d (1-\alpha'_d)} \right)$ and $b' \triangleq \log \left( \frac{1-\theta'}{\theta'} \right)$.
Let 
$$L'(i; \mathcal{C}_k) \triangleq \sum_{d=2}^{W} a'_d \cdot h_d(\{i\}, \mathcal{C}_k) + b' |\Lambda_i(v_k)|,$$ and we further define $L'(i; \mathcal{C}_k, \mathcal{C}_{\bar{k}}) \triangleq L'(i;\mathcal{C}_k) - L'(i;\mathcal{C}_{\bar{k}})$. The following lemma controls the error due to the parameter estimation.

\begin{lemma}\label{lemma7}
    Suppose $p=\Theta \left( \frac{\log m}{n} + \frac{\log n}{m} \right)$ and $m=O(n)$, then for any constant $\tau > 0$, the following statement holds with probability approaching $1$: for any $i \in [n]$, $t>1$ and $\Bar{a} \neq a$,
    \begin{equation}\label{T1S36}
        |L'(i; \mathcal{C}^{(t)}_a, \mathcal{C}^{(t)}_{\Bar{a}}) - L(i; \mathcal{C}^{(t)}_a, \mathcal{C}^{(t)}_{\Bar{a}})| \leq \frac{\tau}{2} \log n.
    \end{equation}
\end{lemma}
\begin{proof}
    See Section~\ref{section5} for the detailed proof.
\end{proof}
By Eqns.~\eqref{T1S34}~\eqref{T1S35}~\eqref{T1S36} and the triangle inequality, we can show that, for any $\{ \mathcal{C}^z_k \}_{k \in [K]} \in \mathcal{Z}_{\delta}$, the output after a single step of iteration belongs to $\mathcal{Z}_{\delta / 2}$. It also means that Stage 3 achieves exact recovery of clusters within $T=O(\log n)$ iterations.

\section{Proof of Theorem~2}\label{section4}
First, note that Theorem~2 can be restated as follows:
\begin{itemize}
    \item \emph{For any $\epsilon > 0$, if $ \sum_{d=2}^{W} \frac{\binom{n-1}{d-1}}{K^{d-1}} I_d + \gamma m p I_\theta \leq (1-\epsilon) \log n$ or $\frac{n}{K} p I_\theta \leq (1-\epsilon) \log m$, then exact matrix completion is impossible.}
\end{itemize}

We first show that the ML estimator is the optimal estimator. Let $\psi_{\text{ML}}|_{\mathcal{R}^{(\gamma)}}$ denote the ML estimator whose output is constrained in $\mathcal{R}^{(\gamma)}$, and let $\widehat{R}$ be the matrix that is chosen uniformly at random from $\mathcal{R}^{(\gamma)}$. We show $ \inf \limits_{\psi} P_{\text{err}}^{(\gamma)}(\psi) \geq \mathbb{P} ( \psi_{\text{ML}}|_{\mathcal{R}^{(\gamma)}} (U, \{ HG_d \}_{d=2}^W) \neq R \ | \ R = \widehat{R} ) $ as follows:
\begin{equation}\label{26}
	\begin{aligned}
		\inf \limits_{\psi} P_{\text{err}}^{(\gamma)}(\psi) &= \inf \limits_{\psi} \max \limits_{X \in \mathcal{R}^{(\gamma)}} \mathbb{P} ( \psi (U, \{ HG_d \}_{d=2}^W) \neq R \ | \ R=X) \\
		& \geq \inf \limits_{\psi} \mathbb{P} ( \psi (U, \{ HG_d \}_{d=2}^W) \neq R \ | \ R=\widehat{R}) \\
		& \overset{\text{(i)}}{=}  \inf \limits_{ \psi: \psi (U, \{ HG_d \}_{d=2}^W) \in \mathcal{R}^{(\gamma)} } \mathbb{P} ( \psi  (U, \{ HG_d \}_{d=2}^W) \neq R \ | \ R=\widehat{R}  )  \\
		& \overset{\text{(ii)}}{=} \mathbb{P} ( \psi_\text{ML}|_{\mathcal{R}^{(\gamma)}} (U, \{ HG_d \}_{d=2}^W) \neq R \ | \ R =\widehat{R}),
	\end{aligned}
\end{equation}
where (i) holds since $\psi (U, \{ HG_d \}_{d=2}^W) \in \mathcal{R}^{(\gamma)} $ should be true for an optimal estimator; (ii) follows from the fact that the ML estimator is the optimal under a uniform prior distribution. Moreover, note that by symmetry, $ \mathbb{P} ( \psi_\text{ML}|_{\mathcal{R}^{(\gamma)}} (U, \{ HG_d \}_{d=2}^W) \neq R \ | \ R =R')$ is identical for any $R' \in \mathcal{R}^{(\gamma)}$, thus one can fix the ground truth matrix to be $R'$ in the following analysis.

Because of the optimality of the ML estimator $\psi_{\text{ML}}|_{\mathcal{R}^{(\gamma)}}$, in order to prove Theorem~2, it suffices to prove $\mathbb{P} ( \psi_\text{ML}|_{\mathcal{R}^{(\gamma)}} (U, \HG) \neq R \ | \ R = R' ) \nrightarrow 0 $ when $ \sum_{d=2}^{W} \frac{\binom{n-1}{d-1}}{K^{d-1}} I_d + \gamma m p I_\theta \leq (1-\epsilon) \log n$ or $\frac{n}{K} p I_\theta \leq (1-\epsilon) \log m$.

Using the log-likelihood function $L(X)$ in Lemma~1, we obtain that
\begin{equation}\label{27}
	\begin{aligned}
		& \mathbb{P} ( \psi_\text{ML}|_{\mathcal{R}^{(\gamma)}} (U, \{HG_d\}_{d=2}^W ) \neq R \ | \ R = R' )\\
		&= \mathbb{P} (\exists X \in \mathcal{R}^{(\gamma)} \setminus \{R'\} \operatorname{ \ s.t. \ } L(X)\geq L(R') ) \\
		& = 1- \mathbb{P} (\forall  X \in \mathcal{R}^{(\gamma)} \setminus \{R'\}: L(X)<L(R')).
	\end{aligned}
\end{equation}
Now we introduce two subsets $\mathcal{R}_1, \mathcal{R}_2 \subseteq \mathcal{R}^{(\gamma)} \setminus \{R'\}$:
\begin{enumerate}[label=(\roman*)]
    \item $\mathcal{R}_1$ is the set of matrices such that the nominal rating vector corresponding to  cluster $\mathcal{C}_1$ is different from the true nominal rating vector $v_1$ in \emph{only one} location. This means that every element $X\in\mathcal{R}_1$ is identical to $R'$ except that the row vectors of $X$ and the row vectors of $R'$ corresponding to cluster $\mathcal{C}_1$ differ in one location.   
    \item $\mathcal{R}_2$ is the set of matrices such that the nominal rating vectors are identical to those of $R'$ but there exists  one user belonging to cluster $\mathcal{C}_1$ and is misclassified to $\mathcal{C}_2$, and there exists another user belonging to cluster $\mathcal{C}_2$ and is misclassified to $\mathcal{C}_1$. Specifically, for every element $X\in\mathcal{R}_2$, the corresponding user clusters, denoted by $\{\mathcal{C}_k^X\}_{k \in [K]}$, satisfy $|\mathcal{C}_1^X \setminus \mathcal{C}_1| = |\mathcal{C}_2^X \setminus \mathcal{C}_2| = 1$ and  $\mathcal{C}_k^X = \mathcal{C}_k$ for $k \in [K]\setminus \{1,2\}$. 
    
\end{enumerate}
Thus, Eqn.~\eqref{27} is lower-bounded by
$$
1- \mathbb{P} (\forall X \in \mathcal{R}_1: L(X)<L(R')) \mathrm{ \ \ \ and \ \ \ } 1-  \mathbb{P} (\forall X \in \mathcal{R}_2: L(X)<L(R')).
$$
Therefore, it suffices to prove that 
\begin{itemize}[wide, labelwidth=!, labelindent=10pt]
    \item $\mathbb{P} (\forall X \in \mathcal{R}_1: L(X)<L(R')) \rightarrow 0$ if  $\frac{n}{K} p I_\theta \leq (1-\epsilon) \log m$ (see Subsection~\ref{part1});
    \item $\mathbb{P} (\forall X \in \mathcal{R}_2: L(X)<L(R')) \rightarrow 0$ if $ \sum_{d=2}^{W} \frac{\binom{n-1}{d-1}}{K^{d-1}} I_d + \gamma m p I_\theta \leq (1-\epsilon) \log n$ (see Subsection~\ref{part2})
\end{itemize}

\subsection{Part 1}\label{part1}
For any $X' \in \mathcal{R}_1$, by using Lemma~\ref{lemma1}, we obtain
\begin{equation}\label{S41_1}
    \mathbb{P} ( L(X') < L(R') ) = 1 - \mathbb{P} ( L(X') \geq L(R') ) = 1- \mathbb{P} \Big( \sum_{i=1}^{n/K} P_i (2\Theta_i -1) \geq 0 \Big),
\end{equation}
where $\{P_i\} \overset{\text{i.i.d}}{\sim} \text{Bern}(p)$ and $\{\Theta_i\} \overset{\text{i.i.d}}{\sim} \text{Bern}(\theta)$.
By applying Lemma~\ref{lemma3} with $L=n/K$, we have 
\begin{equation}
    \mathbb{P} \Big( \sum_{i=1}^{n/K} P_i (2\Theta_i -1) \geq 0 \Big) \geq \frac{1}{4} \exp \Big\{  -(1+o(1))  \frac{n}{K} p I_{\theta} \Big\}.
\end{equation}
Thus, Eqn.~\eqref{S41_1} is upper-bounded by
\begin{equation}
\begin{aligned}
        1 - \frac{1}{4} \exp \Big\{  -(1+o(1))  \frac{n}{K} p I_{\theta} \Big\} &\overset{\text{(i)}}{\leq} \exp \Bigg\{  - \frac{1}{4} \exp \Big\{  -(1+o(1))  \frac{n}{K} p I_{\theta} \Big\} \Bigg\}  \\
        & \overset{\text{(ii)}}{\leq} \exp \Big\{  -\frac{1}{4} m^{\epsilon-1} \Big\},
\end{aligned}
\end{equation}
where (i) follows from $1-x \leq e^{-x}$ and (ii) is due to the condition $\frac{n}{K} p I_\theta \leq (1-\epsilon) \log m$.
Note that for all $X \in \mathcal{R}_1$, the events $\{L(X) < L(R')\}_{X \in \mathcal{R}_1}$ are independent, thus we have
$$
\mathbb{P} (\forall X \in \mathcal{R}_1: L(X)<L(R')) = \mathbb{P} (L(X')<L(R')) ^{|\mathcal{R}_1|} \leq \exp \Big\{  -\frac{1}{4} m^{(\epsilon-1)} \cdot m \Big\} = o(1).
$$

\subsection{Part 2}\label{part2}
First, we use a combinatorial property to split the graphs.

\begin{lemma}\label{lemma8}
	For all $\HG$, we consider the following steps:
	
	(i) Let $r=\frac{n}{\log^3 n}$ and $\mathcal{T} \triangleq \{ 1, 2, ..., 2r \} \cup \{ \frac{n}{K}+1, \frac{n}{K}+2,..., \frac{n}{K}+2r \}. $ 
	
	(ii) For each hyperedge in $\HG$, if it contains two or more user nodes in set $\mathcal{T}$, then we delete these nodes  from $\mathcal{T}$.    
	
	(iii) We define the set of the remaining user nodes as $\mathcal{T}'$.
	
	Let $\Delta$ denote the event $|\mathcal{T}'| \geq 3r$. Then we have $\mathbb{P} (\Delta) = 1-o(1)$.
\end{lemma}

\begin{proof}
    See section~\ref{section5} for the proof.
\end{proof}
By Lemma~\ref{lemma8}, we can find two subsets $\mathcal{C}_1^s \in \mathcal{C}_1$ and $\mathcal{C}_2^s \in \mathcal{C}_2$ that satisfy (i) $|\mathcal{C}_1^s|=|\mathcal{C}_2^s|=\frac{n}{\log^3 n}$ and (ii) there is no hyperedge that contains users in $\mathcal{C}_1^s \cup \mathcal{C}_2^s$. Without loss of generality, we assume $1 \in \mathcal{C}_1^s$ (i.e., the first user belongs to cluster $\mathcal{C}_1^s$).

For the matrix $R'$ and for users $i_1 \in \mathcal{C}^s_1$ and $i_2 \in \mathcal{C}^s_2$, we define $R'^{(i_1)}$ as the matrix that misclassifies user $i_1 \in \mathcal{C}_1$ to $\mathcal{C}_2$, and $R'^{(i_2)}$ as the matrix that misclassifies user $i_2 \in \mathcal{C}_2$ to $\mathcal{C}_1$. Hence, $(R'^{(i_1)})^{(i_2)} \in \mathcal{R}_2$. Since $\mathbb{P}\left( \Delta \right)=1-o(1)$, $\mathbb{P} (\forall X \in \mathcal{R}_2: L(X)<L(R'))$ is upper-bounded by
\begin{equation}\label{S42_1}
        \mathbb{P} \left( \forall \  i_1 \in \mathcal{C}_1^s \mathrm{ \ and \ } i_2 \in \mathcal{C}_2^s:  L\left( \left( R'^{(i_1)} \right)^{(i_2)} \right) < L(R') \right) \cdot (1-o(1)).
\end{equation}

\begin{lemma}\label{lemma9}
	For $i_1 \in \mathcal{C}_1^s$ and $i_2 \in \mathcal{C}_2^s$, if both $L\left( R'^{(i_1)}\right) \geq L(R') $ and $L\left( R'^{(i_2)} \right) \geq L(R')$, then $L\Big( \left( R'^{(i_1)} \right)^{(i_2)} \Big) \geq L(R)$.
\end{lemma}
\begin{proof}
    See Section~\ref{section5} for the detailed proof.
\end{proof}

By Lemma~\ref{lemma9} and the union bound, Eqn.~\eqref{S42_1} is upper-bounded by
\begin{equation}\label{38}
	\mathbb{P} \left( \forall  i_1 \in \mathcal{C}_1^s: L ( R'^{(i_1)} ) < L(R')  \right) \cdot (1-o(1)) + \mathbb{P} \left( \forall  i_2 \in \mathcal{C}_2^s: L ( R'^{(i_2)} ) < L(R')  \right) \cdot (1-o(1)).
\end{equation}

By symmetry, Eqn.~\eqref{38} equals $ 2 \mathbb{P} \left( \forall  i \in \mathcal{C}_1^s: L ( R'^{(i_1)} ) < L(R')  \right) \cdot (1-o(1)) $. Since no pair of users in $\mathcal{C}_1^s$ is connected by any hyperedges, the events $\left\{  L ( R'^{(i_1)} ) < L(R') \right\}_{i_1 \in \mathcal{C}_1^s}$ are mutually independent. Hence,
\begin{equation}\label{39}
	\begin{aligned}
		2 \mathbb{P} \left( \forall  i_1 \in \mathcal{C}_1^s: L ( R'^{(i_1)} ) < L(R')  \right) \cdot (1-o(1)) = 2 \mathbb{P} \left( L(R'^{(1)}) < L(R') \right)^{|\mathcal{C}_1^s|} \cdot (1-o(1)),
	\end{aligned}
\end{equation}
since it is assumed that $1 \in \mathcal{C}_1^s$. According to Lemma~\ref{lemma1}, 
\begin{equation}\label{41}
\mathbb{P} \left( L(R'^{(1)}) < L(R') \right) = 1 -	\mathbb{P} \left( \sum_{d=2}^{W} a_d  \sum_{j=1}^{K_d} (B_{dj} - A_{dj}) + b \sum_{i=1}^{L}P_i(2\Theta_i - 1) \geq 0  \right), 
\end{equation}
where $\{A_{dj}\}_j \overset{\text{i.i.d}}{\sim} \text{Bern}(\alpha_d)$, $\{B_{dj}\}_j \overset{\text{i.i.d}}{\sim} \text{Bern}(\beta_d)$, $\{P_{i}\}_i \overset{\text{i.i.d}}{\sim} \text{Bern}(p)$, and $\{\Theta_{i}\}_i \overset{\text{i.i.d}}{\sim} \text{Bern}(\theta)$, and one can show that $K_d =   \binom{n/K - 1}{d - 1} = (1+ o(1)) \frac{\binom{n-1}{d -1 }}{K^{d-1}}  $ and $L \leq \gamma m$. Then, applying Lemma~\ref{lemma3}, Eqn.~\eqref{41} is upper-bounded by
\begin{equation}\label{42}
\begin{aligned}
    & 1 - \frac{1}{4} \exp \left\{ -(1+o(1)) \sum_{d=2}^{W} \frac{\binom{n-1}{d-1}}{K^{d-1}} I_d - (1+o(1))\gamma m p I_\theta \right\} \\
    & \overset{\text{(i)}}{\leq} \exp \left\{ - \frac{1}{4} \exp \left\{ -(1+o(1)) \sum_{d=2}^{W} \frac{\binom{n-1}{d-1}}{K^{d-1}} I_d - (1+o(1))\gamma m p I_\theta \right\} \right\} \\
    &\overset{\text{(ii)}}{\leq} \exp \Big\{ - \frac{1}{4} n^{\epsilon-1}  \Big\},
\end{aligned}
\end{equation}
where (i) holds since $1-x \leq e^{-x}$, (ii) follows from $ \sum_{d=2}^{W} \frac{\binom{n-1}{d-1}}{K^{d-1}} I_d + \gamma m p I_\theta \leq (1-\epsilon) \log n$.
Hence, Eqn.~\eqref{39} is upper-bounded by
\begin{equation}
    \begin{aligned}
        2 \cdot \exp \Big\{ - \frac{1}{4} n^{\epsilon-1} |\mathcal{C}_1^s| \Big\} \cdot (1-o(1)) = 2 \cdot \exp \Big\{ - \frac{1}{4}  \frac{n^{\epsilon}}{\log^3 n} \Big\} \cdot (1-o(1)) = o(1),
    \end{aligned}
\end{equation}
which means $\mathbb{P} (\forall X \in \mathcal{R}_2: L(X)<L(R')) \rightarrow 0$.

\section{Proof of Lemmas}\label{section5}

 \paragraph{Proof of Lemma 1:} 
Since the observations of $U$ and $\HG$ are mutually independent. The likelihood of $X$ can be decomposed as
\begin{equation}\label{12}
\mathbb{P}\left(\{U, \HG\}|R=X \right)=\mathbb{P}\left(U | R=X \right) \prod_{d=2}^{W}\mathbb{P}\left(HG_d | R=X \right).
\end{equation}
Here
\begin{equation}
\mathbb{P}\left(U | R=X \right) = p^{|\Omega|}(1-p)^{nm - |\Omega|} \theta^{|\Omega|-|\Lambda_X|}(1-\theta)^{|\Lambda_X|}, \quad  \text{and} \label{13}
\end{equation}
\begin{equation}\label{14}
\begin{aligned}
    & \mathbb{P}(HG_d | R=X) = \af^{\sum_{k=1}^K h_d( \mathcal{C}_k^X, \mathcal{C}_k^X)} (1-\af)^{{k\binom{n/k}{d}}-\sum_{k=1}^K h_d( \mathcal{C}_k^X, \mathcal{C}_k^X)} \\
    &\qquad\qquad\qquad\qquad \cdot \bt ^{|\mathcal{H}_d| - \sum_{k=1}^K h_d( \mathcal{C}_k^X, \mathcal{C}_k^X)} (1-\bt)^{\binom{n}{d}-k\binom{n/k}{d}-\sum_{k=1}^K h_d( \mathcal{C}_k^X, \mathcal{C}_k^X)}.
\end{aligned}
\end{equation}
By simple calculations we get
\begin{equation}\label{15}
		\begin{aligned}
			L(X) & = \log \left( \mathbb{P}\left(\{U, \HG\}|R=X \right) \right) \\
			& = \log \left( \mathbb{P}\left(U | R=X \right) \prod_{d=2}^{W}\mathbb{P}\left(HG_d | R=X \right) \right) \\
			& = \log \left( \mathbb{P}\left(U | R=X \right) \right) +\sum_{d=2}^{W} \log \left( \mathbb{P}\left(HG_d | R=X \right) \right) \\
			& = \sum_{d=2}^{W} a_d \cdot \sum_{k=1}^K h_d( \mathcal{C}_k^X, \mathcal{C}_k^X) +  b \cdot |\Lambda_X| + C,
		\end{aligned}
	\end{equation}
	where $C$ is independent of the choice of $X$. 
 \hfill $\square$\par

 \paragraph{Proof of Lemma 2:}  
    Using the Chernoff bound, we have
        \begin{align}
        & \mathbb{P} \left( \sum_{d=2}^{W} a_d  \sum_{j=1}^{K_d} (B_{dj} - A_{dj}) + b \sum_{i=1}^{L}P_i(2\Theta_i - 1) \geq -y  \right) \\
        & \leq\inf \limits_{t>0} e^{ty} \cdot \mathbb{E}\left[ \exp \left\{ \sum_{d=2}^{W} t   \sum_{j=1}^{K_d} a_d (B_{dj} - A_{dj}) + t  \sum_{i=1}^{L}b P_i(2\Theta_i - 1) \right\} \right] \\
        & = \inf \limits_{t>0} e^{ty} \ \prod_{d=2}^{W} \mathbb{M}_{1}(t)^{K_d} \cdot \mathbb{M}_{2} (t)^{L}, \label{eq:e17}
        \end{align}
	where $\mathbb{M}_1(t) \triangleq \mathbb{M}_{a_d(B_{d1}-A_{d1})}(t)$ and  $\mathbb{M}_2(t) \triangleq \mathbb{M}_{bP_1(2\Theta_1 - 1)}(t)$. Using the definitions of $a_d$ and $b$, we obtain 
 \begin{align}
     &\mathbb{M}_1(t) = \af \bt + (1-\af)(1-\bt) + (1-\af)\bt \left( \frac{(1-\bt)\af  }{(1-\af )\bt  }  \right)^{t} + (1-\bt )\af \left( \frac{(1-\bt )\af  }{(1-\af )\bt  }  \right)^{-t}, \notag \\
     &\mathbb{M}_2(t)  = 1-p + p\theta \left( \frac{1-\theta}{\theta} \right)^t + p(1-\theta)\left( \frac{1-\theta}{\theta} \right)^{-t}. \notag
 \end{align}
	Through simple calculations, it can be demonstrated that
	$\frac{1}{2} = \argmin \limits_{t>0} \mathbb{M}_{1}(t) $ and $\frac{1}{2} = \argmin \limits_{t>0} \mathbb{M}_{2}(t)$. Thus, Eqn.~\eqref{eq:e17} is further upper-bounded by
        \begin{equation} \label{eqn11}
            \begin{aligned}
                &   e^{\frac{1}{2}y} \cdot \prod_{d=2}^{W} \mathbb{M}_{1} \left(1/2\right)^{K_d} \cdot \mathbb{M}_{2} \left(1/2 \right)^{L}  \\
			& = \exp \left\{\frac{1}{2}y+ \sum_{d=2}^{W} K_d \log  \mathbb{M}_{1}\left(\frac{1}{2} \right) + L \log \mathbb{M}_{2} \left(\frac{1}{2} \right)  \right\} \\
			& = \exp \left\{ \frac{1}{2}y + \sum_{d=2}^{W} K_d  \cdot 2\log \left(\sqrt{\af \bt } + \sqrt{(1-\af) (1-\bt)} \right) + L \log \left( 2p\sqrt{(1-\theta)\theta}+1-p \right)  \right\} \\
			& \overset{\text{(i)}}{=} \exp\left\{\frac{1}{2}y + \sum_{d=2}^{W} K_d \cdot 2\log \left(\sqrt{ \af \bt } + \left( 1- \frac{1}{2} \af +O(\af^2) \right)\left( 1- \frac{1}{2}\bt + O(\bt^2) \right)\right) \right\} \\
            & \qquad\qquad\qquad\qquad\qquad\qquad\qquad\qquad\qquad\qquad\qquad\qquad \cdot \exp \left\{ L \log \left( 2p\sqrt{(1-\theta)\theta}+1-p \right) \right\} \\
			& \overset{\text{(ii)}}{=} \exp \left\{\frac{1}{2}y + \sum_{d=2}^{W} K_d \cdot 2 \left( \sqrt{\af \bt}-\frac{1}{2}\af-\frac{1}{2}\bt +O(\af^2 + \bt^2) \right) +L \left( p(2\sqrt{(1-\theta)\theta}-1)+O(p^2)\right)  \right\} \\
			& =\exp \left\{ \frac{1}{2}y - \sum_{d=2}^{W} K_d \left((\sqrt{\af}-\sqrt{\bt})^2 + O(\af^2 + \bt^2)\right) - L \left(p(\sqrt{1-\theta}-\sqrt{\theta})^2+O(p^2)\right) \right\} \\
			& =\exp\left\{ \frac{1}{2}y - \sum_{d=2}^{W}(1+o(1))K_d I_d - (1+o(1)) L p I_\theta \right\},
            \end{aligned}
        \end{equation}
	where (i) and (ii) follow from the facts that $\sqrt{1-x}=1-\frac{1}{2}x+O(x^2)$ and $\log(1+x)=x+O(x^2)$ as $x\rightarrow 0$. 
    
 \hfill $\square$\par

\paragraph{Proof of Lemma 3:} 
     Let $Y_{dj} \triangleq a_d ( B_{dj}-A_{dj} )$ and $Z_{k} \triangleq b P_k(2\Theta_k -1)$. The distribution of $Y_{dj}$ is denoted by $p_{Y_{dj}}(\cdot)$, which is identical to $p_{Y_{d1}}(\cdot)$ since $\{Y_{dj}\}_{j=1}^{K_d}$ are independent and identically distributed random variables. Similarly, we denote the distribution of $Z_k$ by $p_{Z_k}(\cdot)$, which is identical to $p_{Z_1}(\cdot)$.  For any $\xi > 0$, we have
	\begin{align*}
		& \mathbb{P} \left( \sum_{d=2}^{W} a_d  \sum_{j=1}^{K_d} (B_{dj} - A_{dj}) + b \sum_{i=1}^{L}P_i(2\Theta_i - 1) \geq 0  \right) \\
		& = \sum_{\{y_{d j}\}, \{z_{k}\}: \sum_{d,j} y_{dj} + \sum_{k} z_k > 0 } \quad \prod_{d=2}^{W} \prod_{j=1}^{K_d} p_{Y_{d1}}(y_{dj}) \prod_{k=1}^{L} p_{Z_1}(z_k) \\
		& \geq \sum_{\{y_{d j}\}, \{z_{k}\}: \sum_{d,j} y_{dj} + \sum_{k} z_k < \xi } \quad \prod_{d=2}^{W} \prod_{j=1}^{K_d} p_{Y_{d1}}(y_{dj}) \prod_{k=1}^{L} p_{Z_1}(z_k) \\
		& \overset{\text{(i)}}{\geq}  \frac{\left(\prod_{d=2}^{W} \mathbb{M}_{Y_{d1}}(1/2)^{K_d}\right) \mathbb{M}_{Z_1}(1/2)^{L} }{e^{\frac{1}{2}\xi}} \\
  & \qquad \qquad \qquad \qquad \qquad \cdot \sum_{\{y_{d j}\}, \{z_{k}\}: \sum_{d,j} y_{dj} + \sum_{k} z_k < \xi } \quad \prod_{d=2}^{W} \prod_{j=1}^{K_d} \frac{e^{\frac{1}{2} y_{dj}} p_{Y_{d1}}(y_{dj})}{\mathbb{M}_{Y_{d1}}(1/2)}  \prod_{k=1}^{L}  \frac{e^{\frac{1}{2} z_k} p_{Z_1}(z_k)}{\mathbb{M}_{Z_1}(1/2)} 
        \end{align*}
        \begin{equation}\label{pl3e1}
            \begin{aligned}
                	& = \exp \left\{ \sum_{d=2}^{W} K_d \log \mathbb{M}_{Y_{d1}}(1/2) + L \log \mathbb{M}_{Z_1}(1/2) - \frac{1}{2} \xi \right\} \cdot \\
		& \quad \quad \sum_{\{y_{d j}\}, \{z_{k}\}: \sum_{d,j} y_{dj} + \sum_{k} z_k < \xi } \quad \prod_{d=2}^{W} \prod_{j=1}^{K_d} \frac{e^{\frac{1}{2} y_{dj}} p_{Y_{d1}}(y_{dj})}{\mathbb{M}_{Y_{d1}}(1/2)}  \prod_{k=1}^{L}  \frac{e^{\frac{1}{2} z_k} p_{Z_1}(z_k)}{\mathbb{M}_{Z_1}(1/2)} \\
		&   \overset{\text{(ii)}}{=} \exp \left\{ \sum_{d=2}^{W} K_d \log \mathbb{M}_{Y_{d1}}(1/2) + L \log \mathbb{M}_{Z_1}(1/2) - \frac{1}{2} \xi \right\} 
		\mathbb{P}\left( 0< \sum_{d=2}^{W} \sum_{j=1}^{K_d} V_{dj} + \sum_{k=1}^{L} W_k < \xi \right), 
            \end{aligned}
        \end{equation}
where (i) holds since $e^{\frac{1}{2} \left( \sum_{d,j} y_{dj} + \sum_{k} z_k \right)} \leq e^{\frac{1}{2} \xi}$ when $\sum_{d,j} y_{dj} + \sum_{k} z_k < \xi$; at (ii), we define new i.i.d. random variables  $\{V_{dj}\}_{j=1}^{K_d}$ with distribution $p_{V_{d1}}(x) = \frac{e^{\frac{1}{2} x} p_{Y_{d1}}(x)}{\mathbb{M}_{Y_{d1}}(1/2)}$, and   $\{W_k\}_{k=1}^L$ with distribution $p_{W_1}(x) = \frac{e^{\frac{1}{2} x} p_{Z_1}(x)}{\mathbb{M}_{Z_1}(1/2)}$.
By Eqn.~\eqref{eqn11}, we get 
\begin{align}\label{pl3e2}
	\exp \left\{ \sum_{d=2}^{W} K_d \log \mathbb{M}_{Y_{d1}}(1/2) + L \log \mathbb{M}_{Z_1}(1/2) - \frac{1}{2} \xi \right\}  \\
 = \exp \left\{ -\frac{1}{2}\xi - \sum_{d=2}^{W}(1+o(1)) K_d I_d - (1+o(1)) L p I_\theta \right\} 
\end{align}

Next, we prove that for a suitable value of $\xi$, 
$$
\mathbb{P} \left( 0< \sum_{d=2}^{W} \sum_{j=1}^{K_d} V_{dj} + \sum_{k=1}^{L} W_k < \xi \right)< \frac{1}{4}.
$$

By Eqn.~\eqref{eqn11}, we have $\mathbb{M}_{Y_{d1}} = \left(\sqrt{\af \bt} + \sqrt{(1-\af)(1-\bt)} \right)^2$ and $ \mathbb{M}_{Z_1}(1/2)=2p\sqrt{(1-\theta)\theta} + 1 - p $. Then, the distribution of ${V_{d1}}$ and $W_1$ equals: 
$$
\begin{aligned}
	&\mathbb{P}\left(V_{d1}=\log \left( \frac{(1-\bt)\af}{(1-\af)\bt} \right)\right) = \mathbb{P}\left(V_{d1}= -\log \left( \frac{(1-\bt)\af}{(1-\af)\bt} \right)\right) = \frac{\sqrt{(1-\af)(1-\bt)\af \bt }}{\left(\sqrt{\af \bt} + \sqrt{(1-\af)(1-\bt)} \right)^2}, \\
	& \mathbb{P} \left( V_{d1} = 0 \right) = \frac{\af \bt + (1-\af)(1-\bt)}{\left(\sqrt{\af \bt} + \sqrt{(1-\af)(1-\bt)} \right)^2} 
\end{aligned} 
$$ and 
$$
\begin{aligned}
	&\mathbb{P}\left(W_1 = \log \left(\frac{1-\theta}{\theta}\right) \right) = \mathbb{P}\left(W_1 = - \log \left(\frac{1-\theta}{\theta}\right) \right) = \frac{p\sqrt{\theta(1-\theta)}}{2p\sqrt{\theta(1-\theta)+ 1+p}}, \\
	& \mathbb{P}(W_1=0) = \frac{1-p}{2p\sqrt{\theta(1-\theta)+ 1+p}}.
\end{aligned}
$$

Thus, simple calculations yield 
\begin{equation} \label{pl3e3}
	\begin{aligned}
		&\mathbb{E}[V_{d1}] = \mathbb{E}[W_1] = 0 \\
		&\mathbb{E}[V_{d1}^2] = \left(\log \left( \frac{(1-\bt)\af}{(1-\af)\bt} \right) \right)^2 \cdot \frac{\af \bt + (1-\af)(1-\bt)}{\left(\sqrt{\af \bt} + \sqrt{(1-\af)(1-\bt)} \right)^2} = O\left( \sqrt{\af \bt} \right) \\
		& \mathbb{E}[W_1^2] = \left( \log \left( \frac{1-\theta}{\theta} \right) \right)^2 \cdot \frac{p\sqrt{\theta(1-\theta)}}{2p\sqrt{\theta(1-\theta) + 1-p}} = O(p) \\
		& \mathbb{E}\left[ \left( \sum_{d=2}^{W} \sum_{j=1}^{K_d} V_{dj} + \sum_{k=1}^{L} W_k \right)^2 \right] = \sum_{d=2}^{W}\sum_{j=1}^{K_d} \mathbb{E}[V_{dj}^2] + \sum_{k=1}^{L}\mathbb{E}[W_k^2] =  \sum_{d=2}^{W} K_d \mathbb{E}[V_{d1}^2] + L\mathbb{E}[W_1^2] \\
            &= \sum_{d=2}^{W}  O\left( \sqrt{\af \bt K_d} \right) + O(pL).
	\end{aligned}
\end{equation}

Let $\xi=  \max \left\{ pL, \sqrt{\af \bt}K_d \right\} ^{3/4}$. By Eqn.~\eqref{pl3e3} and the Chebyshev’s inequality, we have
\begin{equation}\label{pl3e4}
	\begin{aligned}
		&\mathbb{P} \left( 0 < \sum_{d=2}^{W} \sum_{j=1}^{K_d} V_{dj} + \sum_{k=1}^{L} W_k <  \max \left\{ pL, \sqrt{\af \bt}K_d \right\} ^{3/4} \right) \\
		& = \frac{1}{2} - \mathbb{P} \left( \left(\sum_{d=2}^{W} \sum_{j=1}^{K_d} V_{dj} + \sum_{k=1}^{L} W_k\right)^2 \geq \max \left\{ pL, \sqrt{\af \bt}K_d \right\} ^{3/2} \right) \\
		& \geq \frac{1}{2} - \frac{\sum_{d=2}^{W}\sum_{j=1}^{K_d} \mathbb{E}[V_{dj}^2] + \sum_{k=1}^{L}\mathbb{E}[W_k^2]}{ \max \left\{ pL, \sqrt{\af \bt}K_d \right\} ^{3/2}} \\
		& = \frac{1}{2} - \frac{\sum_{d=2}^{W}  O\left( \sqrt{\af \bt K_d} \right) + O(pL)}{ \max \left\{ pL, \sqrt{\af \bt}K_d \right\} ^{3/2}} = \frac{1}{2} - \frac{O( \max \left\{ pL, \sqrt{\af \bt}K_d \right\})}{ \max \left\{ pL, \sqrt{\af \bt}K_d \right\} ^{3/2}} \overset{\text{(i)}}{\rightarrow} \frac{1}{2}> \frac{1}{4},
	\end{aligned}
\end{equation}
where (i) follows from the fact that $\max \left\{ \sqrt{\af \bt}K_d, pL \right\}=\omega(1)$. Substituting Eqn.~\eqref{pl3e2} and Eqn.~\eqref{pl3e4} into Eqn.~\eqref{pl3e1}, we obtain the lower bound in Eqn.~\eqref{eqn4}.

\hfill $\square$\par

 \paragraph{Proof of Lemma 4:} 
    According to the definition of $Y \sim \text{Binom}(\tau n, p)$, we have
	\begin{align}
	\mathbb{P}\left( Y \geq \frac{cnp}{\log \frac{1}{\tau}} \right) = \sum_{i \geq \frac{cnp}{\log \frac{1}{\tau}}} \mathbb{P}(Y=i) = \sum_{i \geq \frac{cnp}{\log \frac{1}{\tau}}} \binom{\tau n}{i} p^i (1-p)^{\tau n - i}. \label{eq:tag}
	\end{align}
	Due to the inequalities $\binom{a}{b}\leq \left( \frac{ea}{b} \right)^b$ and $1-a \leq e^{-a}$, the right-hand side of Eqn.~\eqref{eq:tag} is further upper-bounded by
	$$
	\begin{aligned}
		& e^{-\tau n p}\sum_{i \geq \frac{cnp}{\log \frac{1}{\tau}}} \left( \frac{e \tau n}{i} \right)^i p^i (1-p)^{-i}  \overset{\text{(i)}}{\leq} \sum_{i \geq \frac{cnp}{\log \frac{1}{\tau}}} \left( \frac{2e \tau np}{i} \right)^i \leq \sum_{i \geq \frac{cnp}{\log \frac{1}{\tau}}} \left( \frac{2e \tau np}{\frac{cnp}{\log \frac{1}{\tau}}} \right)^i = \sum_{i \geq \frac{cnp}{\log \frac{1}{\tau}}} \left( \frac{2e \tau \log \frac{1}{\tau}}{c} \right)^i \\
		 & \overset{\text{(ii)}}{\leq} \sum_{i \geq \frac{cnp}{\log \frac{1}{\tau}}} \left( \frac{2e \sqrt{\tau}}{c} \right)^i \overset{\text{(iii)}}{\leq} 2\left( \frac{2e \sqrt{\tau}}{c} \right)^{\frac{cnp}{\log \frac{1}{\tau}}} = 2\exp \left( -\log \left( \frac{c}{2e \sqrt{\tau}} \right) \frac{cnp}{\log \frac{1}{\tau}} \right) \overset{\text{(iv)}}{\leq} 2\exp \left( - \frac{cnp}{2} \right),
	\end{aligned}
	$$
	where (i) follows from the fact $1-p \geq 1/2$; (ii) holds since $\tau \log \frac{1}{\tau} \leq \sqrt{\tau}$ when $0 < \tau \leq 1$; (iii) follows due to the inequality $\sum_{i \geq b} a^i \leq \frac{a^b}{1-a} \leq 2a^b$ for $0<a<1/2$; (iv) holds since $c\geq 2e$.
 \hfill $\square$\par

\paragraph{Proof of Lemma 5:} 
We first prove that for a specific user $i \in [n]$ that belongs to cluster $a \in [K]$, the statement 
    $$
    \mathbb{P} \left( L(i; \mathcal{C}_a) - L(i; \mathcal{C}_{\Bar{a}}) \leq \tau \log n  \right) = O(n^{-1-\epsilon/2})
    $$
holds for any $\Bar{a} \in [K]$ such that $\bar{a}\neq a$. 
According to the expression of local log-likelihood function in Eqn.~\eqref{T1S31}, and recall that $a_d = \ad$ and $b = \bb$, we have
    \begin{equation}
        \begin{aligned}
            &\mathbb{P} \Big( L(i; \mathcal{C}_a) - L(i; \mathcal{C}_{\Bar{a}}) \leq \tau \log n  \Big) \\
            & = \mathbb{P} \Bigg( \sum_{d=2}^W a_d \cdot (h_d(\{i\}, \mathcal{C}_a) - h_d(\{i\}, \mathcal{C}_{\Bar{a}}))  + b ( |\Lambda(v_a)| - |\Lambda(v_{\Bar{a}})| )  \leq \tau \log n \Bigg)\\
            & \overset{\text{(i)}}{\leq} \mathbb{P} \Bigg( \sum_{d=2}^W  a_d \sum_{j=1}^{\binom{n/K-1}{d-1}} (B_{dj} - A_{dj}) + \sum_{j=1}^{\binom{n/K}{d-1}-\binom{n/K-1}{d-1}} B'_{dj}   +  b \sum_{j=1}^{\gamma m} P_i(2\Theta_i - 1) \geq -\tau \log n \Bigg)\\
            & \overset{\text{(ii)}}{\leq} \mathbb{P} \Bigg(  \sum_{d=2}^W a_d \sum_{j=1}^{\binom{n/K-1}{d-1}} (B_{dj} - A_{dj})  + b \sum_{j=1}^{\gamma m} P_i(2\Theta_i - 1) \geq -\tau \log n - o(\log n)  \Bigg),
        \end{aligned}
    \end{equation}
    where $\{A_{dj}\}_j \overset{\text{i.i.d}}{\sim} \text{Bern}(\af)$, $\{B_{dj}\}_j, \{B'_{dj}\}_j \overset{\text{i.i.d}}{\sim} \text{Bern}(\bt)$, $\{P_i\}_i \overset{\text{i.i.d}}{\sim} \text{Bern}(p)$ and $\{\Theta_i\}_i \overset{\text{i.i.d}}{\sim} \text{Bern}(\theta)$. Moreover, (i) follows from the fact $\gamma m = \min_{i,j\in [K]: i \ne j}\Vert v_{i} - v_{j} \Vert_0$; (ii) is true since $\sum_{d=2}^W \sum_{j=1}^{\binom{n/K}{d-1}-\binom{n/K-1}{d-1}} B'_{dj} = o(\log n)$.
        According to Lemma~\ref{lemma2}, the above expression is further upper-bounded by 
    $$
    \begin{aligned}
        & \exp \left( \frac{1}{2} \tau \log n + \frac{1}{2}o(\log n) - (1+o(1))\sum_{d=2}^{W}\frac{\binom{n-1}{d-1}}{K^{d-1}} I_d - (1+o(1)) \gamma m p I_\theta \right) \\
	& \overset{\text{(i)}}{\leq} \exp \left( \frac{1}{2} \tau \log n - (1+o(1)) (1+\epsilon) \log n \right) = n^{-1 - \epsilon + \frac{1}{2} \tau},
    \end{aligned}
    $$
    where (i) follows from the fact $\lim_{n\rightarrow \infty} \binom{n/k-1}{d-1} = \frac{\binom{n-1}{d-1}}{K^{d-1}}$ and the condition $\sum_{d=2}^W \frac{\binom{n-1}{d-1}}{K^{d-1}} I_d + \gamma m p I_{\theta} \geq (1+\epsilon) \log n$. By choosing $\tau$ to be sufficiently small (e.g. $\tau = \epsilon$) and then taking a union bound over all the $n$ users, we complete the proof of Lemma~\ref{lemma5}.
\hfill $\square$\par

\paragraph{Proof of Lemma 6:} 
    For a fixed $\{ \mathcal{C}_k^z \}_{k \in [K]} \in \mathcal{Z}_\delta$, we say user $i$ is \emph{bad} if there exist an $\Bar{a} \neq a$ satisfying $|L(i; \mathcal{C}_a^z, \mathcal{C}_{\Bar{a}}^z) - L(i; \mathcal{C}_a, \mathcal{C}_{\Bar{a}}) | > \frac{\tau}{2} \log n$. Since there are at most $\binom{n}{\delta n}\cdot K^{\delta n}$ many partitions in $\mathcal{Z}_\delta$, and $\binom{n}{\delta n}\cdot K^{\delta n} \leq n^{\delta n}\cdot K^{\delta n} \leq e^{K\delta n \log n}$, it suffices to prove
    $$
	\mathbb{P}\left( \sum_{i=1}^{n} \mathbbm{1} \{ i \text{ \ is bad} \} > \frac{\delta}{2} n \right) \leq O(e^{-(K+1))\delta n \log n}).
    $$
    As the events $\{ \mathbbm{1} \{\text{user} \  i \text{ \ is bad} \}\}_{i=1}^{n}$ are \emph{not} mutually independent, we adopt the technique of \emph{decoupling analysis}~\cite{chen2016community} to handle this issue. First, note that
    \begin{equation}\label{s64}
            \begin{aligned}
            &|L(i; \mathcal{C}_a^z, \mathcal{C}_{\Bar{a}}^z)-L(i; \mathcal{C}_a, \mathcal{C}_{\Bar{a}})| \\
            & = \sum_{d=2}^{W} a_d \cdot |h_d(\{i\}, \mathcal{C}_a^z) - h_d(\{i\}, \mathcal{C}_a) -h_d(\{i\}, \mathcal{C}_{\Bar{a}}^z) + h_d(\{i\}, \mathcal{C}_{\Bar{a}})| \\
            & = \sum_{d=2}^{W} a_d \cdot  |h_d(\{i\}, \mathcal{C}_a^z \setminus \mathcal{C}_a) - h_d(\{i\}, \mathcal{C}_a \setminus \mathcal{C}_a^z) - h_d(\{i\}, \mathcal{C}_{\Bar{a}}^z \setminus \mathcal{C}_{\Bar{a}}) + h_d(\{i\}, \mathcal{C}_{\Bar{a}} \setminus \mathcal{C}_{\Bar{a}}^z)| \\
            & \leq \sum_{d=2}^{W}  a_d \cdot  \big( h_d(\{i\}, \mathcal{C}_a^z \triangle \mathcal{C}_a) + h_d(\{i\}, \mathcal{C}_{\Bar{a}}^z \triangle \mathcal{C}_{\Bar{a}} ) \big) \\
            &\leq 2 \sum_{d=2}^{W}  a_d \cdot  \max \Big\{ h_d(\{i\}, \mathcal{C}_a^z \triangle \mathcal{C}_a), h_d(\{i\}, \mathcal{C}_{\Bar{a}}^z \triangle \mathcal{C}_{\Bar{a}})\Big\},
            \end{aligned}
	\end{equation}
where $\triangle$ represents the \emph{symmetric difference} of two sets. Without loss of generality, we assume $\max \{ h_d(\{i\}, \mathcal{C}_a^z \triangle \mathcal{C}_a), h_d(\{i\}, \mathcal{C}_{\Bar{a}}^z \triangle \mathcal{C}_{\Bar{a}})\} = h_d(\{i\}, \mathcal{C}_a^z \triangle \mathcal{C}_a)$, and the other case can be handled in a similar manner. Then we have the following upper bound
$$
|L(i; \mathcal{C}_a^z, \mathcal{C}_{\Bar{a}}^z)-L(i; \mathcal{C}_a, \mathcal{C}_{\Bar{a}})| \leq 2 \sum_{d=2}^{W}  a_d \cdot  h_d(\{i\}, \mathcal{C}_a^z \triangle \mathcal{C}_a),
$$
and the right-hand side  can further be split as
 $$
2 \sum_{d=2}^{W}  a_d \cdot  ((\Delta_i^1)_d + (\Delta_i^2)_d  ),
$$
where $(\Delta_i^1)_d \triangleq h_d(\{i\}, (\mathcal{C}_a^z \triangle \mathcal{C}_a) \cap \{ 1,2,..., i\} )$ and $(\Delta_i^2)_d \triangleq h_d(\{i\}, (\mathcal{C}_a^z \triangle \mathcal{C}_a) \setminus \{1,2,...,i\} )$. Thus, for each $x=1, 2$, the set of variables $\left\{ (\Delta_i^x)_d \right\}_{i=1}^{n}$ is mutually independent. Then, by defining $\mathbb{I}_i^{x} \triangleq \mathbbm{1}\left\{ \sum_{d=2}^{W} a_d (\Delta_i^x)_d > \frac{\tau}{4 } \log n \right\}$ for $x=1, 2$, we have
$$
\begin{aligned}
	\mathbb{P}\left( \sum_{i=1}^{n} \mathbbm{1} \{ i \text{ \ is bad} \} > \frac{\delta}{2} n \right) & \leq \mathbb{P}\left( \left( \sum_{i=1}^{n} \mathbb{I}_i^1 > \frac{\delta}{4} n \right) \cup \left( \sum_{i=1}^{n} \mathbb{I}_i^2 > \frac{\delta}{4} n \right) \right) \\
	& \leq \mathbb{P}\left( \sum_{i=1}^{n} \mathbb{I}_i^1 > \frac{\delta}{4} n \right) + \mathbb{P}\left( \sum_{i=1}^{n} \mathbb{I}_i^2 > \frac{\delta}{4} n \right).
\end{aligned}
$$
In the following, we will prove that $\mathbb{P}\left(\sum_{i=1}^{n} \mathbb{I}_i^1 > \frac{\delta}{4} n\right)  \leq O(e^{-(k+1))\delta n \log n})$, and the other term can be handled in a similar manner. For $\{ \mathcal{C}^z_k \}_{k \in [K]} \in \mathcal{Z}_\delta$, note that $h_d(\{i\}, \mathcal{C}_a \triangle \mathcal{C}_a^z)$ consists of at most $\binom{\delta n}{d -1}$ independent random variables that are either $\text{Bern}(\af)$ or $\text{Bern}(\bt)$ and $\af > \bt$. For any $d$, let $\{A_{di}\} \overset{\text{i.i.d}}{\sim} \text{Bern}(\alpha_d) $. Hence,
\begin{equation}\label{s65}
    \begin{aligned}
        \mathbb{P}\left( \sum_{d=2}^W a_d (\Delta_i^1)_d  > \frac{\tau}{4 } \log n \right) &\leq \mathbb{P}\left( \sum_{d=2}^W a_d h_d(\{i\}, \mathcal{C}_a \triangle \mathcal{C}_a^z) > \frac{\tau}{4 } \log n \right) \\
        &\leq \mathbb{P} \left( \sum_{d=2}^W a_d \sum_{i=1}^{\binom{\delta n}{d-1}} A_{di} > \frac{\tau}{4 } \log n \right) \\
        &\leq \mathbb{P} \left( \bigcup_{d=2}^W \left\{ a_d \sum_{i=1}^{\binom{\delta n}{d-1}} A_{di} > \frac{\tau}{4 (W-1) } \log n \right\} \right) \\
        & \overset{\text{(i)}}{\leq} \sum_{d=2}^W \mathbb{P} \Bigg( a_d \sum_{i=1}^{\binom{\delta n}{d-1}} A_{di} > \frac{\tau}{4 (W-1)  } \log n  \Bigg) 
    \end{aligned}
\end{equation}
where (i) follows from the fact that if $a_d \sum_{i=1}^{\binom{\delta n}{d-1}} A_{di} < \frac{\tau}{4 (W-1)  } \log n$ are true for all $d$, then $\sum_{d=2}^W a_d (\Delta_i^1)_d$ must less than $\frac{\tau}{4 } \log n$.

For any $2 \le d \le W$, a constant $l>0$ is chosen. By using Lemma~4 with $c=\max \left\{ 5e, l \cdot \frac{2 \log n}{n^{d-1} \alpha_d} \right\}$, we obtain
\begin{equation}\label{s66}
	\mathbb{P}\left( \sum_{i=1}^{\binom{\delta n}{d-1}} A_{di} \geq \frac{cn^{d-1}\alpha_d}{\log \frac{1}{\delta^{d-1}}} \right) \leq \mathbb{P}\left(  \sum_{i=1}^{(\delta n)^{d-1}} A_{di} \geq \frac{cn^{d-1}\alpha_d}{\log \frac{1}{\delta^{d-1}}}  \right) \leq  2\exp \left( -\frac{cn^{d-1}\alpha_d}{2} \right) \leq 2n^{-l}.
\end{equation}
Since $\lim_{\delta \rightarrow 0^+} \frac{1}{\log \frac{1}{\delta}} = 0$, there exists a sufficiently small $\delta' > 0$ such that whenever $\delta < \delta'$,
\begin{equation}\label{s67}
	\frac{cn^{d-1}\alpha_d}{\log \frac{1}{\delta^{d-1}}} \leq \frac{\tau}{4  (W-1) a_d } \log n.
\end{equation}
Thus, for $1 \leq i\leq n$ and $\delta < \delta'$,
\begin{equation}\label{s69}
    \begin{aligned}
        \mathbb{P}\left( \mathbb{I}^{1}_i = 1 \right) &= \mathbb{P}\left(  \sum_{d=2}^{W} a_d (\Delta_i^1)_d > \frac{\tau}{4 } \log n  \right) \leq \sum_{d=2}^W \mathbb{P} \Bigg( a_d\sum_{i=1}^{\binom{\delta n}{d-1}} A_{di} > \frac{\tau}{4 (W-1)  } \log n  \Bigg) \\
        & \leq \sum_{d=2}^W 2n^{-l} = 2(W-1)n^{-l}
    \end{aligned}
\end{equation}

By Chernoff-Hoeffding inequality \cite{thomas2006elements}, we get
\begin{equation}\label{s610}
	\mathbb{P}\left( \sum_{i=1}^{n} \mathbb{I}^1_i > \frac{\delta}{4} n  \right) \leq \exp \left( -n \text{D}_{\text{KL}}\left( \frac{\delta}{4} \bigg|\bigg| 2(W-1)n^{-l} \right)  \right).
\end{equation}
Then, by taking a sufficient large value of $l$, we have
\begin{equation}\label{s611}
	\begin{aligned}
		\text{D}_{\text{KL}} \left( \frac{\delta}{4} \bigg|\bigg| 2(W-1)n^{-l} \right) &= \frac{\delta}{4} \log \left( \frac{\frac{\delta}{4}}{2(W-1)n^{-l}} \right) + \left( 1-\frac{\delta}{4} \right) \log \left( \frac{1- \frac{\delta}{4}}{1-2(W-1)n^{-l}} \right) \\
		& \overset{\text{(i)}}{\geq} \frac{\delta}{4} \log \left( \frac{\frac{\delta}{4}}{2(W-1)n^{-l}} \right) + \log \left( 1-\frac{\delta}{4} \right) \\
		& \overset{\text{(ii)}}{=} \frac{\delta}{4} \log \left( \frac{\delta n^{l}}{8(W-1)} \right) - \frac{\delta}{4} + O(\delta^2) \\
		& \overset{\text{(iii)}}{\geq} \frac{\delta}{4} \cdot \{ (l-1)\cdot \log n - (l-1)\cdot \log 8(W-1) - 1 + O(\delta^2) \} \\
  &\geq (k+1)) \delta n \log n
	\end{aligned}
\end{equation} 
where (i) holds when $l>1$ and $2(W-1)n^{-l} \leq n^{-1} \leq \frac{\delta}{4}$; (ii) follows from the fact that $\log (1-x) = -x + O(x^2)$ as $x \rightarrow 0$; and (iii) follows since $\delta \geq n^{-1}$.

Thus, we obtain
$$
\mathbb{P}\left( \sum_{i=1}^{n} \mathbb{I}^1_i > \frac{\delta}{4} n \right) \leq \exp (-(k+1) \delta n \log n).
$$
    
\hfill $\square$\par

\paragraph{Proof of Lemma 7:} For simplicity, we define $\text{deg}(i)_d \triangleq h_d(\{i\}, [n] \setminus \{i\})$, and $\mathcal{U}(i) \triangleq \{ j \in [m]: (i,j) \in \mathcal{U} \}$. Also recall that  $a'_d \triangleq \log \left( \frac{\alpha'_d (1-\beta'_d)}{\beta'_d (1-\alpha'_d)} \right)$ and $b' \triangleq \log \left( \frac{1-\theta'}{\theta'} \right)$.  Note that
    \begin{equation}\label{s612}
    \begin{aligned}
    & |L'(i; \mathcal{C}^{(t)}_a, \mathcal{C}^{(t)}_{\Bar{a}}) - L(i; \mathcal{C}^{(t)}_a, \mathcal{C}^{(t)}_{\Bar{a}})| \\
    & \leq \sum_{d=2}^{W} | a_d - a'_d | \cdot \Big( h_d(\{i\}, \mathcal{C}^{(t)}_a) + h_d(\{i\}, \mathcal{C}^{(t)}_{\Bar{a}}) \Big)  + |b - b'| \cdot \Big( |\Lambda_i(v_{a})| + |\Lambda_i(v_{\Bar{a}})| \Big) \\
    &\leq \sum_{d=2}^{W} 2 | a_d - a'_d | \cdot  \text{deg}(i)_d  + 2|b - b' |\cdot |\mathcal{U}(i)|.
    \end{aligned}
    \end{equation}
For any $2 \le d\le W$, $\text{deg}(i)_d$ is dominated by $\sum_{i=1}^{\binom{n-1}{d-1}} A_{di}$, where $\{A_{di}\}\overset{\text{i.i.d}}{\sim}\text{Bern}(\af)$. Then, by applying a standard large deviation inequality (e.g., the Bernstein's inequality), we have that  for all $t > 0$,
	\begin{equation}\label{s613}
		\mathbb{P}(\text{deg}(i)_d > t) \leq \mathbb{P}\left( \sum_{i=1}^{\binom{n-1}{d-1}} A_{di} > t \right) \leq 2\exp \left( \frac{-\frac{1}{2} t^2}{\binom{n-1}{d-1} \af + t} \right).
	\end{equation}
	Since $\af = \Theta \left( (\log n)/\binom{n-1}{d -1} \right)$, by taking $t = c_1 \log n$ for a sufficiently large $c_1>0$, we obtain 
	$$
	\mathbb{P}\left(\text{deg}(i)_d > c_1 \log n\right) \leq o(n^{-1}).
	$$ 
	Then by taking a union bound over all the users, we can ensure that $\sum_{d=2}^{W} \text{deg}(i)_d \leq c_1 \log n$ for all $i \in [n]$ with high probability. 
	Similarly, due to the assumptions $p=\Theta\left( \frac{\log m}{n} + \frac{\log n}{m} \right)$ and $m=O(n)$, one can prove that there exists a positive constant $c_2$ such that $|\mathcal{U}(i)|\leq c_2 \log n$ holds for all  $i \in [n]$.
	
	Thus, Eqn.~\eqref{s612} is further upper-bounded by 
 $$
 2 \sum_{d=2}^{W} | a_d - a'_d | \cdot c_1 \log n + 2 |b - b' | \cdot  c_2 \log n.
 $$
 Then, it suffices to prove: 
 \begin{enumerate}[label=(\roman*)]
     \item For every $2 \le d \le W$, $| a_d - a'_d | \leq \frac{\tau}{8 (W-1) c_1}$ with high probability;
     \item $|b-b'| \leq \frac{\tau}{8c_2}$ with high probability.
 \end{enumerate}
  Since the proof of (ii) is similar to (i), we only give the proof of (i) here. Note that, for any $d$,
\begin{equation}\label{s614}
        \begin{aligned}
        &| a_d - a'_d | \\
        & \leq \left|\log \frac{\alpha_d'}{\af } \right| + \left|\log \frac{\beta_d'}{\bt } \right| + \left|\log \frac{1-\alpha_d'}{1-\af } \right| + \left|\log \frac{\beta_d'}{\bt } \right| \\
        & = \left| \log \left( 1+\frac{\alpha_d'-\af }{\af } \right) \right| + \left| \log \left( 1+\frac{\beta_d'-\bt }{\bt } \right) \right| + \left| \log \left( 1-\frac{\alpha_d' - \af  }{1-\af } \right) \right| + \left| \log \left( 1-\frac{\beta_d' - \bt }{1-\bt } \right) \right|.
    \end{aligned}
\end{equation}
Here, we introduce another lemma to complete the proof.
\begin{lemma}\label{lemma10}
		Let $\eta \triangleq \max_{k \in [K]} \frac{|\mathcal{C}^{(t)}_k \setminus \mathcal{C}_k|}{n}$. For a sufficiently small $\eta$, both $\left| \frac{\alpha_d'-\af }{\af } \right| = O(\eta)$ and $\left| \frac{\beta_d'-\bt }{\bt  } \right| = O(\eta)$ holds with high probability.
	\end{lemma}
        \begin{proof}
            We provide the proof after finishing the proofs of Lemmas~\ref{lemma8} and~\ref{lemma9}.
        \end{proof}
	According to Lemma~\ref{lemma10}, Eqn.~\eqref{s614} is upper-bounded by $O(\eta)$. Since Stage 1 guarantees $\eta = O(1)$, the proof of Lemma~\ref{lemma7} is completed.

\hfill $\square$\par

\paragraph{Proof of Lemma 8:}
For any hypergraph $HG_d$, let $\{A_{di}\}_i \overset{\text{i.i.d}}{\sim} \text{Bern}(\alpha_d) $, $\{B_{di}\}_i \overset{\text{i.i.d}}{\sim} \text{Bern}(\beta_d)$  be the set of Bernoulli variables. Then, we denote the set of nodes that are deleted in step (ii) by $\mathcal{F}$. Since $\af \leq \bt$, we have 
$$|\mathcal{F}| \leq \sum_{d=2}^W  d \sum_{i=1}^{\binom{4r}{d}} A_{di}.$$
Hence, by the Markov’s inequality,
\begin{equation}\label{36}
\begin{aligned}
    \mathbb{P} \left( |\mathcal{U}_d| \geq 3r \right) &= 1- \mathbb{P} \left( |\mathcal{F}| \geq r \right) \geq 1- \frac{\mathbb{E}[|\mathcal{F}|]}{r} \geq 1- \frac{\mathbb{E}[ \sum_{d=2}^W d\sum_{i=1}^{\binom{4r}{d}} A_{di}]}{r} \\
    &= 1- \frac{\sum_{d=2}^W d\binom{4r}{d} \alpha_d}{\frac{n}{\log^3 n}} = 1-o(1).
\end{aligned}
\end{equation}
This completes the proof.
\hfill $\square$\par

 \paragraph{Proof of Lemma 9:} 
To prove Lemma~\ref{lemma9}, it suffices to show that 
$$L((R'^{(i_1)})^{(i_2)}) - L(R') \geq L(R'^{(i_1)}) - L(R') + L(R'^{(i_2)}) - L(R'). $$
Let $\mathcal{C}_1^{i_1 i_2}$ represent the cluster that is identical to cluster $\mathcal{C}_1$ except that user $i_1$ (which belongs to $\mathcal{C}_1$) is removed  while user $i_2$ (which belongs to $\mathcal{C}_2$) is added. Similarly,  let $\mathcal{C}_2^{i_1 i_2}$ represent the cluster that is identical to cluster $\mathcal{C}_2$ except that user $i_2$ is removed while user $i_1$ is added.  By Lemma 1, we have
    \begin{equation}
        \begin{aligned}
            &L\Big((R'^{(i_1)})^{(i_2)}\Big) - L(R') \\
            &= \sum_{d=2}^W a_d \Big( h_d( \mathcal{C}_1^{i_1i_2},  \mathcal{C}_1^{i_1i_2}) - h_d( \mathcal{C}_1, \mathcal{C}_1 ) + h_d( \mathcal{C}_2^{i_1i_2},  \mathcal{C}_2^{i_1i_2}) - h_d( \mathcal{C}_2, \mathcal{C}_2 ) \Big) + b ( |\Lambda_{(R'^{(i_1)})^{(i_2)}}| - |\Lambda_{R'}| ).
        \end{aligned}
    \end{equation}
    
    By the definition of $|\Lambda_{X}|$, we have
	\begin{equation}\label{p6_2}
		|\Lambda_{(R'^{(i)})^{(j)}}| - |\Lambda_{R'}| = \left(|\Lambda_{R'^{(i)}}| - |\Lambda_{R'}|\right) + \left(|\Lambda_{R'^{(j)}}| - |\Lambda_{R'}|\right).
	\end{equation}

Let $\mathcal{C}_1^{i_1}$ be the cluster that is identical to $\mathcal{C}_1$ except that user $i_1$ is removed, and $\mathcal{C}_2^{i_1}$ be the cluster that is identical to $\mathcal{C}_2$ except that user $i_1$ is added. The clusters $\mathcal{C}_2^{i_1}$ and $\mathcal{C}_2^{i_2}$ are defined similarly. 
 Note that
    $$
    \begin{aligned}
        &h_d( \mathcal{C}_1^{i_1i_2},  \mathcal{C}_1^{i_1i_2}) = h_d( \mathcal{C}_1^{i_1},  \mathcal{C}_1^{i_1}) + h_d( \{ i_2 \}, \mathcal{C}_1^{i_1} ) , \\
        &h_d( \mathcal{C}_2^{i_1i_2},  \mathcal{C}_2^{i_1i_2}) = h_d( \mathcal{C}_1^{i_1},  \mathcal{C}_1^{i_1}) + h_d( \{ i_1 \}, \mathcal{C}_2^{i_1} ) , \\
        &L (R'^{(i_1)}) - L(R') = \sum_{d=2}^W a_d \Big( h_d( \mathcal{C}_1^{i_1},  \mathcal{C}_1^{i_1}) - h_d( \mathcal{C}_1, \mathcal{C}_1 ) \Big) + b ( |\Lambda_{R'^{(i_1)}}| - |\Lambda_{R'}| ), \\
        & L (R'^{(i_2)}) - L(R') = \sum_{d=2}^W a_d \Big( h_d( \mathcal{C}_2^{i_2},  \mathcal{C}_2^{i_2}) - h_d( \mathcal{C}_2, \mathcal{C}_2 ) \Big) + b ( |\Lambda_{R'^{(i_2)}}| - |\Lambda_{R'}| ).
    \end{aligned}
    $$
    Hence, we have
    $$
    \begin{aligned}
            &L((R'^{(i_1)})^{(i_2)}) - L(R') -\Big( L(R'^{(i_1)}) - L(R') + L(R'^{(i_2)}) - L(R') \Big) \\
            &= \sum_{d=2}^W a_d \Big( h_d( \{ i_2 \}, \mathcal{C}_1^{i_1} ) + h_d( \{ i_1 \}, \mathcal{C}_2^{i_2} ) \Big) > 0.
    \end{aligned}
    $$
 This completes the proof.
  \hfill $\square$\par

 \paragraph{Proof of Lemma 10:} 
    Here, we will only prove $\left| \frac{\alpha_d'-\af }{\af } \right| = O(\eta)$, since the proof of $\left| \frac{\beta_d'-\bt }{\bt } \right| = O(\eta)$ follows similarly. Note that
\begin{equation}\label{ab1}
	h_d(\mathcal{C}_k^{(t)}, \mathcal{C}_k^{(t)}) = \sum_{i=1}^{\binom{(1/K-\eta)n}{d}+\binom{\eta n}{d}} A_i + \sum_{i=1}^{\binom{n/K}{d}-\binom{(1/K-\eta)n}{d}-\binom{\eta n}{d}} B_i,
\end{equation}
where $\{A_i\} \overset{\text{i.i.d}}{\sim} \text{Bern}(\af)$ and $\{B_i\} \overset{\text{i.i.d}}{\sim} \text{Bern}(\bt)$. We define $\gamma \triangleq \frac{\binom{n/K}{d}-\binom{(1/K-\eta)n}{d}-\binom{\eta n}{d}}{\binom{n/K}{d}}$ for simplicity, and one can show that $\gamma = O(\eta)$. Then, by the triangle inequality, we have
\begin{equation}\label{ab2}
	\left| \frac{1}{K} \af - \frac{h_d(\mathcal{C}_k^{(t)}, \mathcal{C}_k^{(t)})}{K\binom{n/K}{d}}  \right| \leq \left| \frac{(1-\gamma)}{K} \af - \frac{1}{K\binom{n/K}{d}} \sum_{i=1}^{(1-\gamma)\binom{n/K}{d}} A_i \right| + \frac{\gamma}{K}\af + \frac{1}{K\binom{n/K}{d}} \sum_{i=1}^{\gamma \binom{n/K}{d}} B_i.
\end{equation}

By Chernoff-Hoeffding inequality, for any $l>0$, we have
\begin{equation}\label{ab3}
    \begin{aligned}
        &\mathbb{P} \left( \left| \frac{(1-\gamma)}{K} \af - \frac{1}{K\binom{n/K}{d}} \sum_{i=1}^{(1-\gamma)\binom{n/K}{d}} A_i \right| \geq  \frac{(1- \gamma) }{K} l \af \right) \\
        &\qquad  \qquad \qquad \qquad \qquad \qquad \le \exp \left( -(1-\gamma) \binom{n/K}{d} D_{KL} \left(\af +\gamma \af \bigg| \bigg| \af\right) \right).
    \end{aligned}	
\end{equation}
Since $ D_{KL} \left(\af +\gamma \af \bigg| \bigg| \af\right) = \Theta(l \af)$ and $\af = \Theta \left( \frac{\log n}{\binom{n-1}{d-1}} \right)$, Eqn.~\eqref{ab3} is upper-bounded by
\begin{equation}\label{ab4}
    \begin{aligned}
        &\mathbb{P} \left( \left| \frac{(1-\gamma)}{K} \af - \frac{1}{K\binom{n/K}{d}} \sum_{i=1}^{(1-\gamma)\binom{n/K}{d}} A_i \right| \geq  \frac{1- \gamma }{K} l \af \right) \\
        &\qquad \qquad \qquad \qquad \qquad \qquad \qquad \qquad  \le \exp \left( -(1-\gamma) \binom{n/K}{d} l  \frac{\log n}{\binom{n-1}{d-1}} \right)  = o(n^{-1}).
    \end{aligned}	
\end{equation}
Setting $l=\frac{K \gamma}{1-\gamma}$, we have
\begin{equation}\label{ab5}
	\mathbb{P} \left( \left| \frac{(1-\gamma)}{K} \af - \frac{1}{K\binom{n/K}{d}} \sum_{i=1}^{(1-\gamma)\binom{n/K}{d}} A_i \right|  \leq \gamma \af \right) = 1- o(1).
\end{equation}
By applying Bernstein's inequality, the following statement holds with high probability,
\begin{equation}\label{ab6}
	\frac{1}{K\binom{n/K}{d}} \sum_{i=1}^{\gamma \binom{n/K}{d}} B_i \leq K \bt = O(\af).
\end{equation}
Substitute \eqref{ab5} and \eqref{ab6} into \eqref{ab2}, we obtain $\left| \frac{1}{K} \af - \frac{h_d(\mathcal{C}_k^{(t)}, \mathcal{C}_k^{(t)})}{K\binom{n/K}{d}}  \right| = O(\eta) \af$.  Then we have  
$$
\left| \af - \alpha_d' \right| \leq \sum_{k \in [K]} \left| \frac{1}{K} \af - \frac{h_d(\mathcal{C}_k^{(t)}, \mathcal{C}_k^{(t)})}{K\binom{n/K}{d}}  \right|   = O(\eta) \af,
$$
which completes the proof of $\left| \frac{\alpha_d'-\af }{\af } \right| = O(\eta)$.
 \hfill $\square$\par

\section{Additional Experiments}\label{experiments}

\subsection{Incorporate Hyperedges to Graph-based Methods}
To further explore the MCH algorithm's ability to utilize hyperedge information, we conduct additional experiments. Specifically, we employ clique expansion to convert hyperedges into fully connected edges (thereby transforming hypergraphs into graphs), enabling the utilization of additional hyperedge information in graph-based methods. The results in Figure~\ref{fig_4} indicate that incorporating hyperedge information does enhance the performance of graph-based baselines. However, the proposed MCH algorithm continues to outperform them due to its ability to fully utilize the hypergraph information.

\subsection{Running-Time Comparison}
We compare the running times of the selected $8$ algorithms, as shown in Table~\ref{Tab1}. Our proposed MCH exihibits high efficiency. Note that all the experiments  are performed with the same hardware setup, including a 6-core i7 9750H CPU and 16GB of memory.

\begin{figure}[h]
    \centering
    \begin{subfigure}[b]{0.36\linewidth}
        \includegraphics[width=\linewidth]{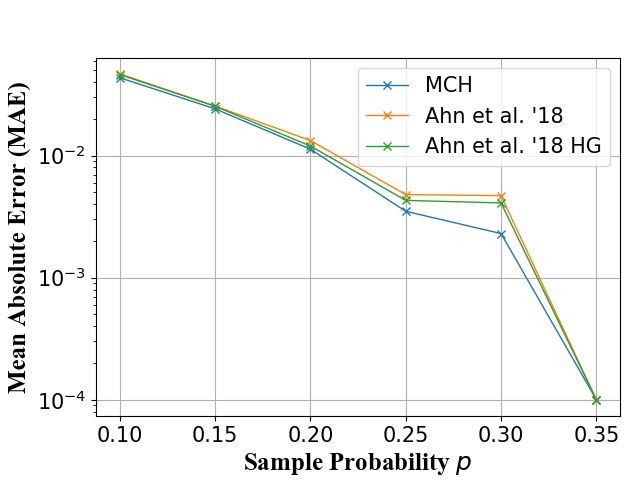}
        \caption{Semi-real data}
    \end{subfigure}
    \hspace{1em}
     \begin{subfigure}[b]{0.36\linewidth}
        \includegraphics[width=\linewidth]{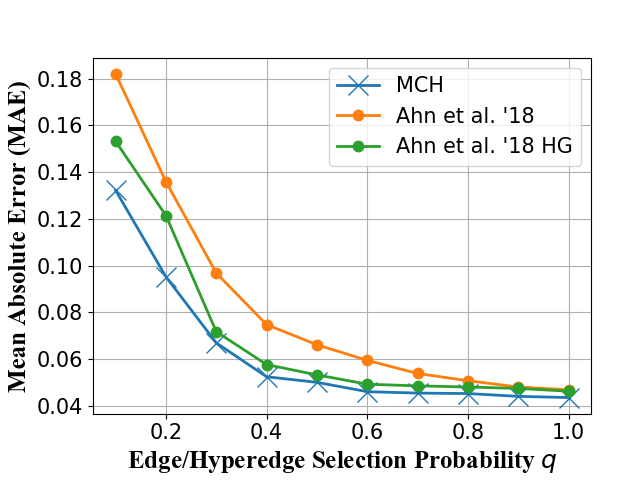}
        \caption{Modified semi-real data (fix $p$ = 0.1)}
    \end{subfigure}
    \vspace{-0.5em}
     \begin{subfigure}[b]{0.36\linewidth}
        \includegraphics[width=\linewidth]{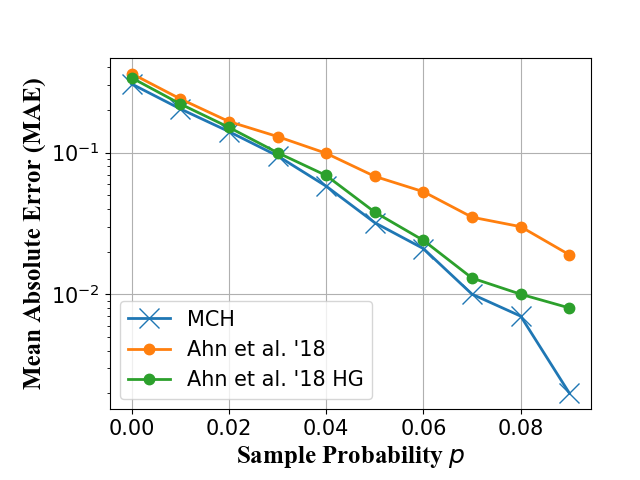}
        \caption{Synthetic data ($n=300$, $m=100$)}
    \end{subfigure}
    \caption{Comparative Experiments: Hyperedge Inclusion (Ahn et al. HG) vs. Exclusion (Ahn et al.) on the synthetic dataset and semi-real dataset.}
    \label{fig_4}
\end{figure}

\begin{table}[h]
  \centering
  \caption{Comparison of runng times for different algorithms. Each algorithm undergoes 10 iterations.}
  \label{Tab1}
  \begin{tabular}{c | c c c c c}
    \hline
    Methods & GraphRec & socialMF & soialReg & svd++ & user k-NN \\
    \hline
    Running Time (/s) & 291.93 & 301.33 & 33.28 & 22.52 & 3.15 \\
    \hline
    Methods & item k-NN & Ahn et al. '18 & MCH \\
    \hline
    Running Time (/s) & 3.45 & 5.3 & 5.7 \\
    \hline
  \end{tabular}
\end{table}

\section{A detailed comparison with~\cite{ahn2018binary}}\label{compare}
Ref.~\cite{ahn2018binary} was the first to theoretically investigate how much gain graph information can provide for matrix completion problems. Our work is inspired by this study, with the main influences reflected in the following aspects:  
\begin{enumerate}
    \item \textbf{Framework Adoption:} We adopted the same research framework as ~\cite{ahn2018binary}, focusing on the SBM model for matrix completion problems involving binary ratings. 
    \item \textbf{Algorithm Design:} Drawing inspiration from ~\cite{ahn2018binary}, we also employ a similar three-stage algorithm.
\end{enumerate} 
However, our work introduces the following three key improvements:  
\begin{enumerate}
    \item \textbf{Theoretical Assumptions:} While ~\cite{ahn2018binary} considers the case of only two symmetric culsters, we extend this assumption to multiple symmetric culsters. Furthermore, our proposed MCH algorithm is capable of handling multiple asymmetric culsters in practical applications, addressing the most general scenarios.
    \item \textbf{Problem Setting:} The primary distinction lies in the problem setting. We consider the presence of multiple social \textit{hypergraphs}, whereas ~\cite{ahn2018binary} only addresses a single social graph. In addition, our treatment of multiple uniform hypergraphs can be viewed as a non-uniform hypergraph. This extension necessitates solving more complex combinatorial problems and deriving tighter bounds to achieve sharp thresholds—an inherently \textit{non-trivial} challenge.
    \item \textbf{Experiments and Applications:} The MCH algorithm is designed to handle hypergraph social information, while the algorithm in ~\cite{ahn2018binary} reduces hypergraph information to graph information for processing. To demonstrate the benefits of leveraging hypergraph information, we conducted experiments, including those shown in Figures~\ref{R2} and the three experiments in Figure~\ref{fig_4}. These experiments confirm that using hypergraph information provides additional gains, validating the practical significance of our approach.  
\end{enumerate}

\medskip

\end{document}